\documentclass[11pt]{article} %
\usepackage{newtxtext}
\usepackage[left=1.25in, top=1in, bottom=1in, right=1.25in]{geometry}

\usepackage{amsmath,amsfonts,bm}

\def\1{\bm{1}}

\DeclareMathAlphabet{\mathsfit}{\encodingdefault}{\sfdefault}{m}{sl}
\SetMathAlphabet{\mathsfit}{bold}{\encodingdefault}{\sfdefault}{bx}{n}

\newcommand*{\E}{\mathbb{E}}

\usepackage{xcolor}
\definecolor{NavyBlue}{rgb}{0.1, 0.4, 0.8}
\usepackage[colorlinks, 
            linkcolor=NavyBlue,
            urlcolor=gray,
            citecolor=NavyBlue,
            anchorcolor=NavyBlue,
            backref=page
]{hyperref}

\usepackage{url}       
\usepackage{natbib}

\usepackage{caption}
\usepackage{graphicx}
\usepackage{amsmath,amsthm}
\usepackage{amssymb}
\usepackage{dsfont}
\usepackage{booktabs}
\usepackage[para]{footmisc}

\renewcommand\cite{\citep}	% to get "(Author Year)" with natbib    
\usepackage{colortbl}

\newcommand{\TCR}[1]{\textcolor{red}{#1}}
\newcommand{\TCG}[1]{\textcolor{green}{#1}}

\newtheorem{theorem}{Theorem}

\newtheorem{proposition}{Proposition}
\newtheorem{lemma}{Lemma}
\newtheorem{definition}{Definition}
\DeclareMathOperator{\A}{\mathcal{A}}
\DeclareMathOperator{\M}{\mathcal{M}}
\DeclareMathOperator{\LL}{\mathcal{L}}
\usepackage{float} 

\newcommand{\thmstable}{If the loss function $\ell$ is $\rho-$Lipschitz, $\A(S^i)$ is close to $\A(S)$, the Hessian matrix $\nabla^2\mathcal{L}(\A(S))$ at $\A(S)$ is  positive-semidefinite with a singular value decomposition $U\operatorname{diag}(\Lambda) U^{-1}$, $\Lambda=\{\Lambda_1,\cdots,\Lambda_m\}$ and $\Lambda_{min}=\min\{\Lambda_1,\cdots,\Lambda_m\}$, then the expectation of the loss $\E_M L_R$ has a pointwise hypothesis stability as:
\begin{equation}
  \E_{S,i\sim U(n)}[|\ell(\A(S^{i}),z_i)-\ell(\A(S),z_i)|]\le \frac{2\rho^2 }{(\Lambda_{min}+2(1-p))n}.
\end{equation}}

\newcommand{\thmgeneralization}{We denote the generalization error as $R(\A,S)=\E_z \ell(\A(S),z)$ and the empirical error as $\hat{R}(\A,S)=\frac{1}{n}\sum_{i=1}^n \ell(\A(S),z)$. Then, for some constant $C$, we have with probability $1-\delta$,
\begin{equation}
  R(\A,S)\le \hat{R}(\A,S)+\sqrt{\frac{C^2+\frac{24C\rho^2}{\Lambda_{min}+2(1-p)}}{2n\delta}}.
\end{equation}
}

\newcommand{\thmtoreg}{Optimizing Problem (\ref{eqn:primal}) implies to optimizing the upper bound $\bar{L}$ of the following regularized problem:
\begin{equation}
  L_R= \min_\theta \mathcal{L}(\theta)+\|(I-M)(\theta-\theta^0)\|^2 \le \bar{L}.
\end{equation}
}

\newcommand{\thmapprox}{If $\hat{M}_{ii}=\mathds{1}(\sum_{j=1}^m\mathds{1}(|\frac{\nabla \LL(\theta^0)_i^2}{h_i}| > |\frac{\nabla \LL(\theta^0)_j^2}{h_j}|)\ge m-\lfloor mp\rfloor)$, where $\nabla \LL(\theta^0)_i$ is the $i$th element of the gradient vector $\nabla \LL(\theta^0)$, then
\begin{equation}
  \inf_{\Delta\theta} \LL(\theta^0+\hat{M}\Delta\theta)\le \inf_{\substack{\Delta\theta,\|M\|_0=\lfloor mp\rfloor;\\M_{ij}=0,\forall i\ne j; M_{ii}\in \{0,1\}}}  \LL(\theta^0+M\Delta\theta).
\end{equation}
}

\title{\textbf{On the Effectiveness of Parameter-Efficient Fine-Tuning}}

\author{
    \textbf{Zihao Fu,\textsuperscript{\rm 1} Haoran Yang,\textsuperscript{\rm 2}
    Anthony Man-Cho So,\textsuperscript{\rm 2}}\\
    \textbf{Wai Lam,\textsuperscript{\rm 2}
    Lidong Bing,\textsuperscript{\rm 3}
    Nigel Collier\textsuperscript{\rm 1}}
}

\date{\textsuperscript{\rm 1}Language Technology Lab, University of Cambridge,\\
\textsuperscript{\rm 2}The Chinese University of Hong Kong,
\textsuperscript{\rm 3}DAMO Academy, Alibaba Group\\
\{zf268,nhc30\}@cam.ac.uk,\\
    \{hryang,manchoso,wlam\}@se.cuhk.edu.hk,
    l.bing@alibaba-inc.com
}

\begin{document}

\maketitle

\begin{abstract}
  Fine-tuning pre-trained models has been ubiquitously proven to be effective in a wide range of NLP tasks. However, fine-tuning the whole model is parameter inefficient as it always yields an entirely new model for each task. Currently, many research works propose to only fine-tune a small portion of the parameters while keeping most of the parameters shared across different tasks. 
  These methods achieve surprisingly good performance and are shown to be more stable than their corresponding fully fine-tuned counterparts. However, such kind of methods is still not well understood. Some natural questions arise: How does the parameter sparsity lead to promising performance? Why is the model more stable than the fully fine-tuned models? How to choose the tunable parameters?
  In this paper, we first categorize the existing methods into random approaches, rule-based approaches, and projection-based approaches based on how they choose which parameters to tune. Then, we show that all of the methods are actually sparse fine-tuned models and conduct a novel theoretical analysis of them. We indicate that the sparsity is actually imposing a regularization on the original model by controlling the upper bound of the stability. Such stability leads to better generalization capability which has been empirically observed in a lot of recent research works. Despite the effectiveness of sparsity grounded by our theory, it still remains an open problem of how to choose the tunable parameters. Currently, the random and rule-based methods do not utilize task-specific data information while the projection-based approaches suffer from the projection discontinuity problem. To better choose the tunable parameters, we propose a novel Second-order Approximation Method (SAM) which approximates the original problem with an analytically solvable optimization function. The tunable parameters are determined by directly optimizing the approximation function. 
  We conduct extensive experiments on several tasks. The experimental results\footnote{The code is available at\ \ \ \   \url{https://github.com/fuzihaofzh/AnalyzeParameterEfficientFinetune}} show that our proposed SAM model outperforms many strong baseline models and it also verifies our theoretical analysis.

\end{abstract}

\section{Introduction}
Fine-tuning the model parameters for a specific task on a pre-trained model \cite{Peters2018DeepCW,kenton2019bert,lan2020albert,radford2018improving,radford2019language,liu2019roberta,brown2020language,lewis2020bart,raffel2020exploring} has become one of the most promising techniques for NLP in recent years. It achieves state-of-the-art performance on most of the NLP tasks. However, as the parameter number grows exponentially to billions \cite{brown2020language} or even trillions \cite{fedus2021switch}, it becomes very inefficient to save the fully fine-tuned parameters \cite{he2021towards} for each downstream task. Many recent research works propose a parameter-efficient \cite{houlsby2019parameter,zaken2021bitfit,he2021towards} way to solve this problem by tuning only a small part of the original parameters and storing the tuned parameters for each task.

Apart from the efficiency of the parameter-efficient models, it has also been observed in many recent research works that the parameter-efficient methods achieve surprisingly good performance. These models are more stable  \cite{he2021effectiveness,lee2019mixout,houlsby2019parameter,zaken2021bitfit,sung2021training,liu2021p,ding2022delta}  and even achieve better overall scores than the fully fine-tuned models \cite{lee2019mixout,houlsby2019parameter,zaken2021bitfit,sung2021training,liu2021p,xu2021raise,guo2021parameter,he2021towards,ding2022delta} on some tasks. Currently, it remains unclear why the parameter-efficient models can improve the stability and performance in many prevalent works. In this paper, we first categorize the existing methods into three categories (i.e. random approaches, rule-based approaches, and projection-based approaches) depending on how they choose the tunable parameters. Then, we define the generalized sparse fine-tuned model and illustrate that most of the existing parameter-efficient models are actually a sparse fine-tuned model. Afterwards, we introduce the widely used pointwise hypothesis stability of the sparse fine-tuned model and show theoretically that the sparsity actually controls the upper bound of the stability. Based on the stability analysis, we further give a theoretical analysis of the generalization bound for the sparse fine-tuned~model. 

Though promising results have been achieved by existing parameter-efficient models, it still remains a challenging problem to select suitable parameters as it is an NP-hard problem. Currently, the random \cite{lee2019mixout} and rule-based \cite{zaken2021bitfit,han2015deep,houlsby2019parameter,pfeiffer2020adapterhub} approaches propose to optimize fixed parameters. These methods are straightforward and easy to implement but they do not utilize task-specific data information. To solve this problem, the projection-based approaches \cite{mallya2018piggyback,guo2021parameter,xu2021raise} propose to calculate a score for each parameter based on the data and project the scores onto the parameter selection mask's feasible region (an $L_0$ ball). However, as the feasible region is non-convex, we will show that such projection suffers from the projection discontinuity problem which makes the parameter selection quite unstable. To solve these problems, we propose a novel Second-order Approximation Method (SAM) to approximate the NP-hard optimization target function with an analytically solvable function. Then, we directly choose the parameters based on the optimal value and optimize the parameters accordingly. We conduct extensive experiments to validate our theoretical analysis and our proposed SAM model.

Our contributions can be summarized as follows: 1) We propose a new categorization scheme for existing parameter-efficient methods and generalize most of these methods with a unified view called the sparse fine-tuned model. 2) We conduct a theoretical analysis of the parameter-efficient models' stability and generalization. 3) We propose a novel SAM model to choose the suitable parameters to optimize. 4) We conduct extensive experiments to verify our theoretical analysis and the SAM model.

\section{Unified View of Parameter Efficient Fine-tuning}
In this section, we first define the unified sparse fine-tuned model which is simpler and easier for theoretical analysis.  Then, we give a unified form of the optimization target. Afterwards, similar to previous works \cite{ding2022delta,he2021towards,mao2021unipelt}, we categorize these models into three categories based on how the parameters are chosen. Finally, we show that all the models are sparse fine-tuned model.

\subsection{Sparse Fine-tuned Model}
We first give the definition of sparse fine-tuned model as well as a unified optimization target. The equivalent model is also defined to help understand the models with modified structures. 
\begin{definition}[$p$-Sparse Fine-tuned Model]
  Given a pre-trained model $\M^0$ with parameters $\theta^0$, if a fine-tuned model $\M$ with parameters $\theta$ has the same structure as $\M^0$ such that $\|\theta-\theta^0\|_0\le p\dim(\theta),p\in(0,1)$, we say the model $\M$ is a $p$-sparse fine-tuned model with the sparsity $p$.
  \label{def:psparse} 
\end{definition}

Many previous works propose different methods of selecting proper parameters to fine-tune. We unify these methods by denoting $M$ as a mask matrix on the parameters and the parameter $\theta$ can be denoted as $\theta=\theta^0 +M\Delta \theta$, where $\Delta \theta$ is the difference vector. For a fixed sparsity coefficient $p$, the sparse fine-tuned model is trying to solve the following problem:

\begin{equation}
  \small
\begin{aligned}
  &\min_{\Delta \theta,M} \LL(\theta^0 +M\Delta \theta)\\
  s.t. \ \ \ \ \|M\|_0=&\lfloor mp\rfloor;\ \ \ \ M_{ij}=0,\forall i\ne j; \ \ \ \  M_{ii}\in \{0,1\},
\end{aligned}
\label{eqn:sfm}
\end{equation}

where $\lfloor \cdot\rfloor$ is the floor function, $m=\dim(\theta)$ is the parameter number, $M\in\{0,1\}^{m\times m}$ is the parameter mask matrix with the diagonal equal to 0 or 1 while other elements are equal to 0 and $\LL$ is the loss function. We will show that most of the existing methods are sparse fine-tuned models. However, in Definition \ref{def:psparse}, we assume that the fine-tuned model $\M$ has the same structure as $\M^0$. This assumption hinders us from analyzing many models that alter the structure including Adapter \cite{houlsby2019parameter,pfeiffer2020adapterhub,ruckle2021adapterdrop,he2021effectiveness}, LoRA \cite{hu2022lora}, and etc. We define the notion of equivalent model to solve this problem.

\begin{definition}[Equivalent Model] Given a pre-trained model $\M^0$ with parameters $\theta^0$, we say that a model $\tilde{\M^0}$ with parameters $\tilde{\theta}^0$ is an equivalent model for model $\M^0$ if $\forall x , \M^0(x) = \tilde{\M^0}(x)$.
\end{definition}

Here, we do not require that the equivalent model shares the same structure as the original model. As a result, for models fine-tuned with additional structures (e.g. Adapter and LoRA), we can still get a sparse fine-tuned model with respect to an equivalent model $\tilde{\M^0}$ instead of the original pre-trained model $\M^0$. Therefore, our analysis for the sparse fine-tuned model is also applicable to them.

\subsection{Parameter Efficient Fine-tuning as Sparse Fine-tuned Model}\label{sec:models}

\begin{figure}[t]
  \centering
  \centering
  \includegraphics[width=0.9\columnwidth]{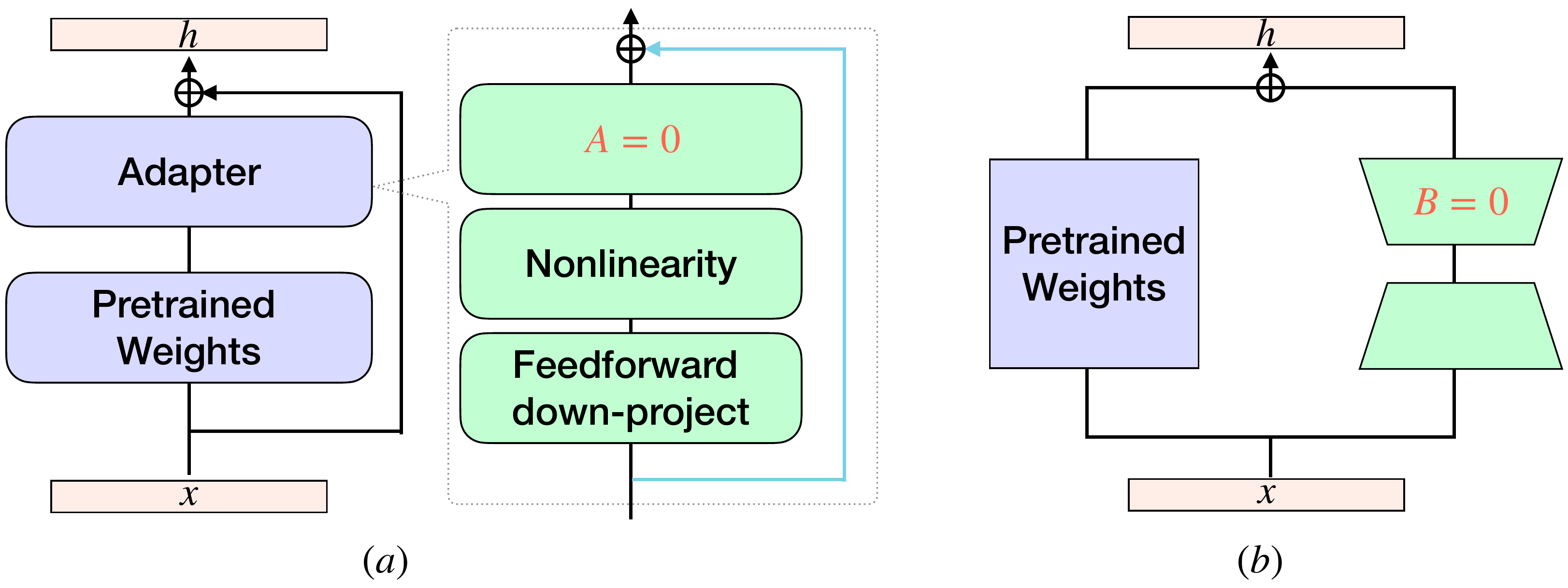}
  \caption{Equivalent model for Adapter (a)  and LoRA (b).}
  \label{fig:lora-and-adapter}
\end{figure}

Unfortunately, Problem (\ref{eqn:sfm}) is NP-Hard due to the nonconvexity of the feasible region of the matrix $M$. Many existing methods propose to solve this problem by first estimating $M$ and then optimizing other parameters. Based on different strategies for choosing $M$, the methods can be divided into three categories, namely, random approaches, rule-based approaches, and projection-based approaches.
We first give a general introduction of the prevalent parameter efficient fine-tuning methods in each category and then show that all of these methods are actually a sparse fine-tuned model. Then, in next section, we can prove our theory only based on properties in Definition \ref{def:psparse} without refering to any specific model's property.

\subsubsection{Random Approaches}
Random approaches include Random and Mixout models. These models randomly choose the parameters to be tuned. Such selection does not depend on the task-specific data information. Specifically,
\textbf{Random} model is very straightforward by randomly selecting the parameters with respect to a given sparsity ratio and then training the selected parameters. Therefore, according to Definition \ref{def:psparse}, it is a sparse fine-tuned model.
\textbf{Mixout} \cite{lee2019mixout} proposes to directly reset a portion of the fine-tuned model's parameters to the pre-trained parameters with respect to a ratio. Therefore, according to Definition \ref{def:psparse}, it is a sparse fine-tuned model.

\subsubsection{Rule-Based Approaches} 
The rule-based approaches include BitFit, MagPruning, Adapter, and LoRA. This kind of methods directly uses a pre-defined rule to fix the parameters to be tuned. It can be viewed as incorporating prior knowledge to recognize important features and can thus alleviate the problem of random approaches. However, the selection rules are still irrelevant to the specific data. Specifically,
\textbf{BitFit} \cite{zaken2021bitfit} only fine-tunes the bias-terms and achieves considerably good performance. Therefore, according to Definition \ref{def:psparse}, it is a sparse fine-tuned model with pre-defined tuning weights.
\textbf{MagPruning} \cite{han2015deep,han2015learning,lee2021layeradaptive,lagunas2021block} follows the idea that large weights are more important in the model. It ranks the weights by the absolute value and tunes the parameters with high absolute values. Therefore, according to Definition \ref{def:psparse}, it is a sparse fine-tuned model.
\textbf{Adapter} \cite{houlsby2019parameter,pfeiffer2020adapterhub,ruckle2021adapterdrop,he2021effectiveness,karimi2021compacter,kim2021revisiting,mahabadi2021parameter} proposes to add an adapter layer inside the transformer layer. Therefore, the model structure is different from the original model. To make it easier to analyze, Adapter can be viewed as fine-tuning an equivalent model shown in Fig. \ref{fig:lora-and-adapter} (a) which initializes the matrix $A$ as an all-zero matrix. The equivalent model has the same output as the original pre-trained model for arbitrary input while the structure is the same as the Adapter model. Therefore, fine-tuning the adapter model can be viewed as fine-tuning partial parameters of the equivalent model with the same structure. According to Definition \ref{def:psparse}, it is a sparse fine-tuned model with respect to the equivalent model.
\textbf{LoRA} \cite{hu2022lora,karimi2021compacter,panahi2021shapeshifter} proposes to add a new vector calculated by recovering an hidden vector from a lower dimension space. The model is illustrated in Fig. \ref{fig:lora-and-adapter} (b). It is interesting to notice that the original initialization makes the LoRA model already an equivalent model for the original pre-trained model as the matrix $B$ is set to 0. Therefore, according to Definition \ref{def:psparse}, fine-tuning a LoRA model can also be viewed as fine-tuning partial parameters of the equivalent model with the same structure.

\subsubsection{Projection-Based Approaches} 
To utilize the task-specific data to help select the model's tunable parameters, many researchers propose projection-based approaches including the DiffPruning, ChildPruning, and etc. These methods propose to choose the optimal parameter mask $M$ and optimize the parameters $\theta$ alternately to solve Problem (\ref{eqn:sfm}). Specifically, they first relax $M$ as a continuous variable to get an optimized value and then project the optimized value onto the feasible region which can be denoted as $\hat{M}=\Pi_\Omega (M)=\arg\min_{\hat{M}\in \Omega}\|\hat{M}-M\|$, where $\Omega=\{M|\|M\|_0=\lfloor mp\rfloor; M_{ij}=0,\forall i\ne j;   M_{ii}\in \{0,1\}\}$ and $\Pi_\Omega$ denotes the projection operator onto the feasible region $\Omega$ which is an $L_0$ ball. Specifically,
\textbf{DiffPruning} \cite{mallya2018piggyback,sanh2020movement,guo2021parameter,lagunas2021block} proposes to model the parameter selection mask as a Bernoulli random variable and optimize the variable with a reparametrization method. It then projects the mask onto $M$'s feasible region $\Omega$ and do the optimization alternately. Therefore, according to Definition \ref{def:psparse}, it is also a sparse fine-tuned model.
\textbf{ChildPruning} \cite{xu2021raise,mostafa2019parameter} proposes to iteratively train the full model parameters and then calculates the projected mask to find the child network. Therefore, it also agrees with the sparse fine-tuned model's definition.

\textbf{Projection Discontinuity Problem}. Though projection-based methods can utilize task-specific data information, such kind of methods suffers from the projection discontinuity problem. Specifically, the feasible region $\Omega$ (the $L_0$ ball) of $M$ is non-convex. Therefore, it does not have the non-expansion property which is generally guaranteed for projection onto a closed convex set. As a result, a small perturbation on $M$ can lead to a totally different projection. For example, as illustrated in Fig. \ref{fig:proj}, suppose that $p=0.5$ and $M_1=\operatorname{diag}\{0.99,1\},M_2=\operatorname{diag}\{1,0.99\}$. Though $M_1\approx M_2$, we have $\Pi_\Omega(M_1)=\operatorname{diag}\{0,1\}$ while $\Pi_\Omega(M_2)=\operatorname{diag}\{1,0\}$, which is quite different. Consequently, the projection is very sensitive to the parameters updating noise. As a result, it is hard to keep consistent with the previous parameters selection which leads to a big change for the parameters selection. Such inconsistency will impair the overall performance.

\begin{figure}[t]
  \centering
    \centering
    \includegraphics[width=0.3\columnwidth]{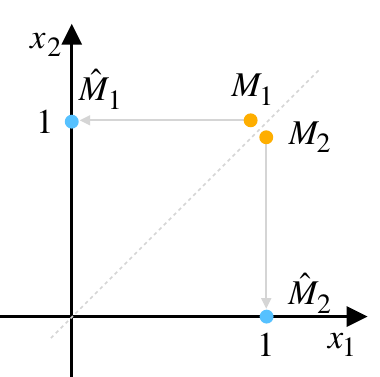}
  \caption{Projection discontinuity problem.}
  \label{fig:proj}
\end{figure}

\section{Theoretical Analysis of the Sparse Fine-tuned Model}
\label{sec:theory}
Suppose that we have a pre-trained model $\M^0$ with parameters $\theta^0$, we fine-tune the sparse fine-tuned model $\M$ by updating only $\lfloor pm\rfloor$ parameters. We will first show that sparsity implies a regularization of the original model. Then, we prove that if a model is a sparse fine-tuned model, the model stability can benefit from the sparsity. Next, we give a theoretical analysis of the model generalization error bound and show that sparsity contributes to reducing the generalization error. It should be noted that in the proofs, we only use properties from Definition \ref{def:psparse}. Therefore, our theory is applicable to all model categories (random approaches, rule-based approaches, and projection-based approaches) that agrees with Definition \ref{def:psparse}.

\subsection{Sparse Fine-tuned Model as a Regularizer}

As analyzed in section \ref{sec:models}, most of the models choose the parameter mask $M$ with different approaches and optimize the parameters $\theta$ accordingly. Here, we treat the matrix $M$ as a given parameter and denote $\theta=\theta^0 +M\Delta \theta$. The sparse fine-tuned optimization in Problem (\ref{eqn:sfm}) can be reformulated~as:

\begin{equation}
  \small
  \begin{aligned}
    &\min_\theta \mathcal{L}(\theta)\\
    s.t.\ \ \|(I&-M)(\theta-\theta^0)\|^2=0,
  \end{aligned}
  \label{eqn:primal}
\end{equation}
where $M=\operatorname{diag}\{M_{11},\cdots,M_{mm}\}$ is a diagonal matrix with $M_{ii}\in\{0,1\}$. By Lagrangian duality, solving Problem (\ref{eqn:primal}) is equivalent to solving the following problem:

\begin{equation}
  \small
  \bar{L}=\min_\theta \max_\lambda \mathcal{L}(\theta)+\lambda\|(I-M)(\theta-\theta^0)\|^2.
\end{equation}

Then, we derive a new regularized problem with the following proposition.

\begin{proposition}
  \small
  \thmtoreg
  
  \label{thm:toreg}
\end{proposition}

The proof can be found in Appendix \ref{sec:A-toreg}. It can be concluded that optimizing Problem (\ref{eqn:primal}) is the same as optimizing the upper bound of the original loss function $\mathcal{L}(\theta)$ with a regularization term $\|(I-M)(\theta-\theta^0)\|^2$. We will show later that such regularization contributes to the stability of the sparse fine-tuned model.

\subsection{Stability Analysis}

Stability has been studied in a lot of previous research works \cite{bousquet2002stability,shalev2010learnability,shalev2014understanding,hardt2016train,kuzborskij2018data,charles2018stability,fu2021theoretical} in many different forms. We focus on one of the commonly used notions, namely, the Pointwise Hypothesis Stability (PHS) which focuses on analyzing the change of model output after a training sample is removed. Following \cite{charles2018stability}, we denote the original training data as $S=\{z_1,\cdots,z_n\}$ and the dataset without one sample as $S^i=S\backslash z_i=\{z_1,\cdots,z_{i-1},z_{i+1},\cdots,z_n\}$, where $z_i$ is the $i$th training sample. We also define $i\sim U(n)$ as a sampling procedure from a uniform distribution with $n$ samples. $\A(S)$ is defined as model parameters obtained by running algorithm $\A$ on data $S$.
\begin{definition}
  [Pointwise Hypothesis Stability, \cite{bousquet2002stability}] We say that a learning algorithm $\A$ has \textbf{pointwise hypothesis stability} $\epsilon$ with respect to a loss function $\ell$, if 
  \begin{equation}
    \small
    \E_{S,i\sim U(n)}[|\ell(\A(S^{i}),z_i)-\ell(\A(S),z_i)|]\le \epsilon.
  \end{equation}

  \label{def:stability}
\end{definition}
Here, $\ell(\theta,z_i)$ is the single sample loss for $z_i$ when the model parameter is $\theta$. We assume that $\A(S^i)$ is close to $\A(S)$. As $\A(S)$ is the optimal solution, the Hessian matrix at $\A(S)$ is a positive-semidefinite matrix. We can derive our bound for PHS in the following theorem.

\begin{theorem}[Stability] 
  \small
  \thmstable

  \label{thm:stable}
\end{theorem}

The proof can be found in Appendix \ref{sec:A-stable}. It can be observed from Theorem \ref{thm:stable} that as the sparsity parameter $p$ decreases, the upper bound also decreases. Therefore, sparse models imply better stability which explains most of the empirical results observed in many recent works \cite{he2021effectiveness,lee2019mixout,houlsby2019parameter,zaken2021bitfit,sung2021training,liu2021p,ding2022delta}. It should also be noted that if $p$ is small enough, the upper bound will not change significantly as $p$ continues to decrease. This is because in this case, the denominator is dominated by $\Lambda_{min}$ which is related to the landscape of the function. Empirically, if the sparsity is too small, the landscape will heavily depend on how the parameters are chosen and thus the stability is impaired.

\subsection{Generalization Analysis}
With the bound for the stability, we can then get the generalization error bound for the sparse fine-tuned model.

\begin{theorem}[Generalization] 
  \small
  \thmgeneralization
  
  \label{thm:generalization}
\end{theorem}
The proof can be found in Appendix \ref{sec:A-generalization}. This result shows that the generalization error upper bound becomes smaller as the fine-tuned parameters become sparser. Intuitively, if a model is stable, a perturbation makes less effect on the model and the model is less likely to overfit.
It should be noted that the generalization error bound is determined by both the empirical error $\hat{R}(\A, S)$ and sparsity. Therefore, as the mask becomes sparser, even though the second term decreases, the training error $\hat{R}(\A, S)$ will possibly increase when the tunable parameters are not enough to fit the data. Consequently, as the sparsity decreases, the generalization error will first decrease and then increase. We will further examine this conjecture in experiments.

\section{Second-order Approximation Method}
In Section \ref{sec:theory}, we theoretically prove the effectiveness of sparsity in fine-tuning. However, it still remains a problem of how to choose the tunable parameters.
As discussed in Section \ref{sec:models}, the random and the rule-based approaches are robust to noise perturbation as the tunable parameters are fixed during training. However, these methods tune the same parameters on all kinds of tasks without utilizing the information from the task-specific data. On the other hand, the projection-based approaches solve this problem by getting full utilization of the data information but they suffer from the projection discontinuity problem. The noise in the parameter may change the selection of the parameters frequently, thus making the optimization procedure unstable. 

To solve the problems, we propose a novel Second-order Approximation Method (SAM), namely, utilizing the data information to help decide the parameter mask while avoiding the projection discontinuity problem. Instead of choosing the parameters randomly or simply by some rules, we propose a novel second-order approximation of Problem (\ref{eqn:sfm}) to make the optimization target analytically solvable. Then, we directly get the optimal solution for the parameter mask $M$ and fix the mask to train the other parameters $\theta$. Specifically, as indicated by \citet{radiya2020fine}, the fine-tuned parameters are close to the pre-trained parameters. We can approximate the loss function with its second-order Taylor expansion as $\mathcal{L}(\theta^0+M\Delta \theta)\approx \LL(\theta^0) + \nabla \LL(\theta^0)^{\mathrm T} M\Delta \theta + \frac{1}{2} (M\Delta \theta)^{\mathrm T} H M\Delta \theta.$
Unfortunately, the Hessian matrix $H$ is expensive to compute especially for a large neural model. To solve this problem, we adopt the widely used technique \cite{dembo1982inexact,ricotti1988learning,bishop2006pattern,bollapragada2019exact,xu2020newton,yao2021inexact,yao2021adahessian} of approximating the Hessian matrix with a diagonal matrix denoted as $H=\operatorname{diag}\{h_1,h_2,\cdots,h_n\}$. We also assume that $H$ is positive semidefinite as the pre-trained weights is close to the global minimizer \cite{radiya2020fine} in each downstream task. Then, Problem (\ref{eqn:sfm}) can be reformulated as:

\begin{equation}
  \small
  \begin{aligned}
    \min_{\Delta \theta} \LL(\theta^0) +& \nabla \LL(\theta^0)^{\mathrm T} M\Delta \theta + \frac{1}{2} (M\Delta \theta)^{\mathrm T} H M\Delta \theta\\
    s.t. \ \ \ \ \|M\|_0=\lfloor mp&\rfloor;\ \ \ \ M_{ij}=0,\forall i\ne j; \ \ \ \  M_{ii}\in \{0,1\}.
  \end{aligned}
  \label{eqn:approx}
\end{equation}

With the above setup, we can get the optimal parameter mask $M$ for Problem (\ref{eqn:approx}) based on the following theorem:

\begin{theorem}
  \small
  \thmapprox
  
  \label{thm:minapprox}
\end{theorem}
The proof can be found in Appendix \ref{sec:A-minapprox}. It can be observed that selecting features according to Theorem \ref{thm:minapprox} achieves the minimal value of the approximation in Problem (\ref{eqn:approx}). The remaining problem is how to calculate the diagonal of the Hessian matrix. Unfortunately, calculating the diagonal Hessian is as complex as calculating the whole Hessian. To solve this problem, instead of minimizing the target function in Problem (\ref{eqn:approx}), we propose to optimize its upper bound

\begin{equation}
  \small
  \begin{aligned}
    \min_{\Delta \theta} \LL(\theta^0) +& \nabla \LL(\theta^0)^{\mathrm T} M\Delta \theta + \frac{1}{2} (M\Delta \theta)^{\mathrm T} D M\Delta \theta\\
    s.t. \ \ \ \ \|M\|_0=\lfloor mp&\rfloor;\ \ \ \ M_{ij}=0,\forall i\ne j; \ \ \ \  M_{ii}\in \{0,1\}.
  \end{aligned}
  \label{eqn:approxub}
\end{equation}

where $D=\text{diag}\{|\lambda_{max}|, |\lambda_{max}|, \cdots, |\lambda_{max}| \}$ and $\lambda_{max}$ is the maximal eigenvalue of $H$. This can be directly calculated from the Rayleigh quotient that $\forall x\ne 0, x^THx\le x^Tx\lambda_{max}\le x^Tx|\lambda_{max}|=x^TDx$. Therefore, the SAM algorithm is quite straightforward based on Theorem \ref{thm:minapprox}. We first get the gradient $\nabla \LL(\theta^0)_i$ for the $i$th parameter $\theta_i$. Then, we calculate $|\nabla \LL(\theta^0)_i^2|$ and take the top $\lfloor mp \rfloor$ parameters to optimize. We will not change the selected parameters during the optimization procedure.

\section{Experiments}  
\subsection{Experimental Setup}
Following most previous works \cite{phang2018sentence,lee2019mixout,dodge2020fine,xu2021raise}, we use the original development set as the test set to report the scores as the original test sets are only available via the leaderboard with a limited submission number. Different from many previous works that train models without validation, we split the original training set by randomly sampling 10\% as the new development set while using the remaining 90\% samples to train the model. Instead of training the model for fixed epoch number, we use the new development set to do an early stop training by setting the tolerance for all models to 40. We build our models with the jiant\footnote{\url{https://jiant.info/}} \cite{phang2020jiant} framework and test our models on several GLUE \cite{wang2018glue} and SuperGLUE \cite{wang2019superglue} tasks. Following the setting of \citet{lee2019mixout,xu2021raise}, we choose several tasks including Corpus of Linguistic Acceptability (CoLA) \cite{warstadt2019neural}, Semantic Textual Similarity Benchmark (STS-B) \cite{cer2017semeval}, Microsoft Research Paraphrase Corpus (MRPC) \cite{dolan2005automatically}, Recognizing Textual Entailment (RTE) \cite{dagan2005pascal,giampiccolo2007third,bentivogli2009fifth}, Commitment Bank (CB) \cite{de2019commitmentbank}, Choice of Plausible Alternatives (COPA) \cite{roemmele2011choice}, and Winograd Schema Challenge (WSC) \cite{levesque2012winograd}. We compare our model with many strong baseline models including Random, Mixout, BitFit, MagPruning, Adapter, LoRA, DiffPruning, and ChildPruning. The details of these models have been extensively discussed in Section \ref{sec:models} and we adopt the same evaluation methods as \citet{wang2018glue,wang2019superglue} to evaluate the models. We run each experiment 10 times with different random seeds and report the scores with corresponding standard deviations. As many previous experiments are conducted under different settings, we re-implement all the baseline models with the jiant framework to give a fair comparison. For the Adapter and LoRA model, we incorporate AdapterHub \footnote{\url{https://adapterhub.ml/}} \cite{pfeiffer2020adapterhub} and loralib \footnote{\url{https://github.com/microsoft/LoRA}} into jiant. Following the setting of \citet{guo2021parameter}, we set the sparsity to 0.005 for all models for a fair comparison. In  SAM, we calculate $\nabla\mathcal{L}(\theta^0)_i$ by accumulating the gradient for a few burn-in steps as we cannot load all the training data into memory, the burn-in steps are chosen from $\{500, 600, 700, 800, 900, 1000, 2000\}$ on the development set as a hyper-parameter. For CoLA and MRPC, we set burn-in step to 600; For STS-B, we set burn-in step to 1000; For RTE, CB, and COPA we set burn-in step to 700; For WSC, we set burn-in step to 2000. We fine-tune the models based on RoBERTa-base \cite{liu2019roberta} provided by transformers\footnote{\url{https://huggingface.co/docs/transformers/model_doc/roberta}} toolkit \cite{wolf-etal-2020-transformers} and we run the models on NVIDIA TITAN RTX GPU with 24GB memory.

\begin{table*}[t]
  \centering
  \scriptsize
  \setlength{\leftskip}{-40pt}
  \begin{tabular}{l@{~}|@{~}l@{~}@{~}l@{~}@{~}l@{~}@{~}l@{~}@{~}l@{~}@{~}l@{~}@{~}l@{~}|@{~}l}
  \toprule
  {} &                CoLA &              STS-B &                MRPC &                 RTE &                 CB &               COPA &                WSC &                 AVG \\
  \hline
  FullTuning   &                           58.36±1.74 &                           89.80±0.52 &  \textbf{89.55}$_{[1]}$±0.81 &                                    76.03±2.14 &           88.93$_{[2]}$±2.37$_{[2]}$ &                                    67.70±4.41 &                           53.10±6.18 &                                    74.78±2.60 \\
Random       &                   58.35±1.05$_{[2]}$ &          89.81±\textbf{0.11}$_{[1]}$ &                   88.73±0.80 &                                    72.71±3.23 &          \textbf{90.54}$_{[1]}$±3.39 &                                    68.80±2.64 &                           52.88±5.97 &                                    74.55±2.46 \\
MixOut       &                           58.66±1.96 &                   90.15$_{[3]}$±0.17 &           88.69±0.60$_{[3]}$ &  \textbf{77.55}$_{[1]}$±\textbf{1.64}$_{[1]}$ &                           86.51±4.13 &                                    71.30±4.84 &                           52.98±6.78 &                            75.12$_{[3]}$±2.88 \\
Bitfit       &                           56.67±1.45 &                   90.12±0.14$_{[3]}$ &           87.35±0.58$_{[2]}$ &                                    72.74±2.47 &                           86.96±3.20 &                                    71.20±3.79 &                           55.10±5.39 &                                    74.31±2.43 \\
MagPruning   &                           56.57±2.47 &           90.30$_{[2]}$±0.14$_{[3]}$ &                   88.09±0.79 &                            73.53±1.84$_{[3]}$ &                           81.25±3.50 &                    71.50$_{[3]}$±2.46$_{[2]}$ &          55.67±\textbf{2.73}$_{[1]}$ &                            73.85±1.99$_{[2]}$ \\
Adapter      &  \textbf{62.11}$_{[1]}$±1.22$_{[3]}$ &                   90.05±0.13$_{[2]}$ &   89.29$_{[3]}$±0.60$_{[3]}$ &                            76.93$_{[3]}$±2.05 &                           87.32±4.62 &                            69.50±2.54$_{[3]}$ &                   57.02$_{[2]}$±5.27 &                            76.03$_{[2]}$±2.35 \\
LoRA         &                   60.88$_{[3]}$±1.48 &                           87.19±0.51 &           89.53$_{[2]}$±0.62 &                            76.97$_{[2]}$±1.92 &                           84.64±3.76 &                                    69.70±2.83 &                   56.84$_{[3]}$±4.52 &                            75.11±2.24$_{[3]}$ \\
DiffPruning  &                           58.53±1.49 &                           89.59±0.34 &                   78.79±6.09 &                                    69.93±7.87 &                   86.25±2.65$_{[3]}$ &                            72.10$_{[2]}$±2.91 &                   53.37±3.60$_{[3]}$ &                                    72.65±3.57 \\
ChildPruning &                           60.00±1.29 &                           89.97±1.51 &                   87.19±3.86 &                                    75.76±4.38 &                           86.61±3.22 &                                    69.40±4.00 &                           55.59±3.81 &                                    74.93±3.15 \\
SAM          &  60.89$_{[2]}$±\textbf{0.96}$_{[1]}$ &  \textbf{90.59}$_{[1]}$±0.14$_{[3]}$ &  88.84±\textbf{0.49}$_{[1]}$ &                            76.79±1.72$_{[2]}$ &  88.93$_{[2]}$±\textbf{1.75}$_{[1]}$ &  \textbf{74.30}$_{[1]}$±\textbf{2.45}$_{[1]}$ &  \textbf{59.52}$_{[1]}$±3.08$_{[2]}$ &  \textbf{77.12}$_{[1]}$±\textbf{1.51}$_{[1]}$ \\
  \bottomrule
  \end{tabular}
  
\captionsetup{width=1.0\textwidth}  

\caption{Main experiment. We run each experiment 10 times with different random seeds and report means and standard deviations. The number in the bracket is the rank for the scores in the corresponding column. Due to the space limit, we attach the training time analysis and the significance test in Appendix \ref{sec:A-timeexp} and \ref{sec:A-ttest}.}

\label{tab:main}
\end{table*}

\subsection{Experimental Results}

\textbf{Main Experiment.} The main experimental results are illustrated in Table \ref{tab:main}. We can draw the following conclusions based on the results: (1) Most of the parameter-efficient models achieve better performance than the FullTuning model which is also consistent with the observations in many previous works. This observation supports our theoretical analysis in Theorem \ref{thm:generalization} that the parameter-efficient model has better generalization capability. (2) Most of the parameter-efficient models are more stable than the FullTuning model. This observation is also consistent with many empirical results in previous works and it also supports our theoretical stability analysis in Theorem \ref{thm:stable}. (3) It is interesting to note that even the Random model outperforms the FullTuning model. It shows that sparsity itself contributes to improving the performance. (4) Our proposed SAM model outperforms several baseline models in several tasks and it ranks in the top 3 of most tasks. This observation validates the effectiveness of our parameter selecting method discussed in Theorem \ref{thm:minapprox}. Due to the space limit, we attach the training time analysis and the significance test in Appendix \ref{sec:A-timeexp} and \ref{sec:A-ttest}.

\begin{figure}[t]
  \makebox[\linewidth][c]{
  \centering
  \begin{minipage}[t]{0.45\textwidth}
  \centering
  \includegraphics[width=0.99\columnwidth]{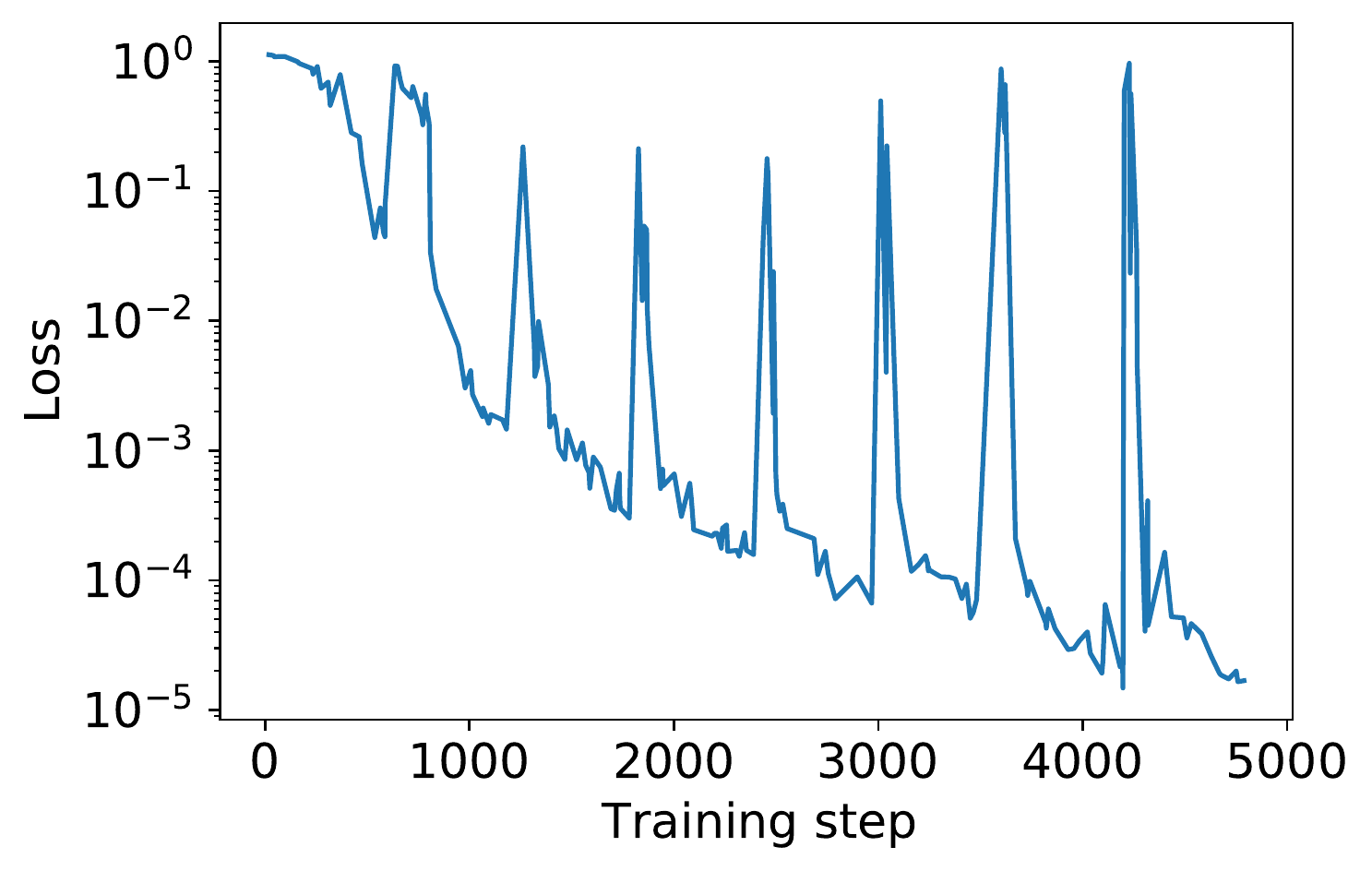}
  \caption{Projection discontinuity problem.}
  \label{fig:cb-sgpa-600}
  \end{minipage}
  \quad
  \begin{minipage}[t]{0.43\textwidth}
    \centering
    \includegraphics[width=0.95\columnwidth]{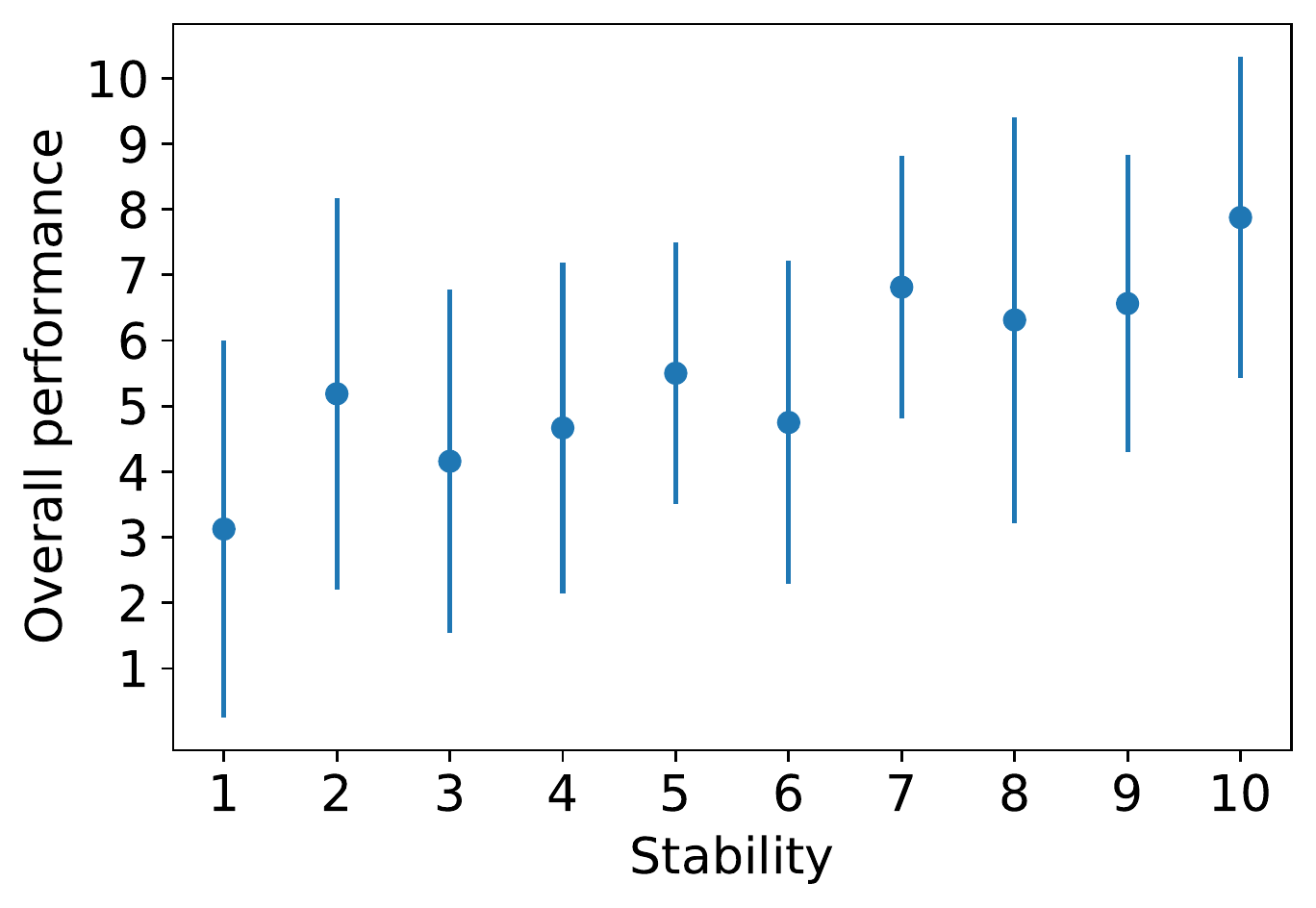}
  \caption{Relation between stability and overall Performance.}
  \label{fig:spearman}
  \end{minipage}
  }
\end{figure}

\begin{figure*}[t]
  \centering
  \includegraphics[width=1.\columnwidth]{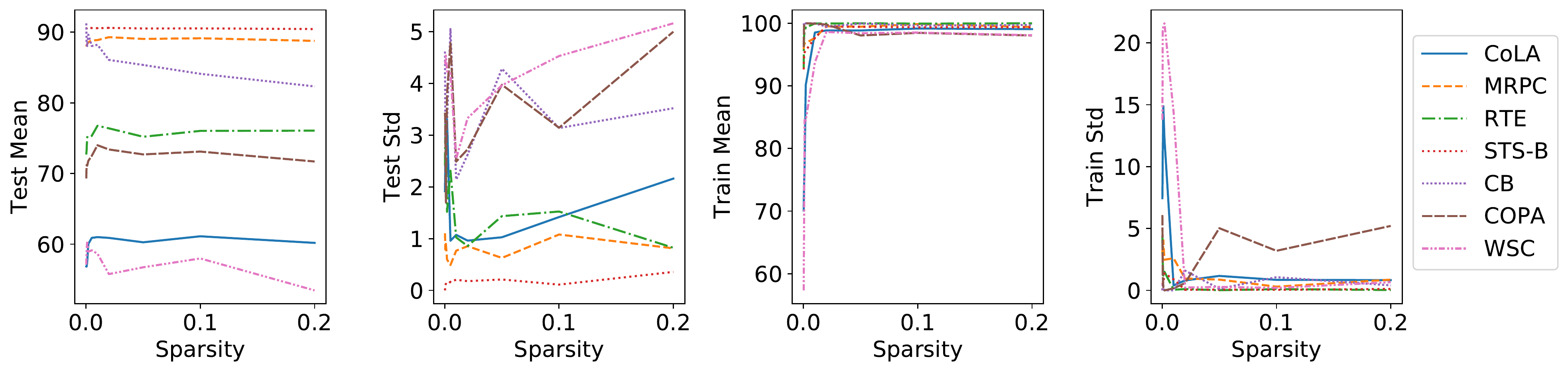}
\caption{Effectiveness of sparsity.}
\label{fig:sparsity}
\end{figure*}

\begin{table*}[t]
  \centering
  \scriptsize
  \setlength{\leftskip}{-40pt}
  \begin{tabular}{l@{~}|@{~}l@{~}@{~}l@{~}@{~}l@{~}@{~}l@{~}@{~}l@{~}@{~}l@{~}@{~}l@{~}|@{~}l}
  \toprule
  {} &               CoLA &               STS-B &                MRPC &                 RTE &                 CB &               COPA &                 WSC &                 AVG \\
  \hline
  FullTuning   &                   60.74$_{[2]}$±1.89 &                   90.11$_{[3]}$±0.26 &                            88.74$_{[3]}$±1.08 &                   75.37$_{[3]}$±1.93 &                                    84.29±4.21 &                   69.60±2.94 &                           54.81±7.51 &                                    74.81±2.83 \\
Random       &                           56.00±1.84 &                           89.79±0.20 &                            88.57±0.72$_{[2]}$ &                           73.00±2.01 &                            89.29$_{[2]}$±4.92 &           70.30±2.69$_{[3]}$ &                           56.87±4.29 &                                    74.83±2.38 \\
MixOut       &                   60.37$_{[3]}$±1.33 &           90.11$_{[3]}$±0.13$_{[3]}$ &                            88.50±0.78$_{[3]}$ &                   74.51±1.28$_{[2]}$ &                            83.75±3.14$_{[3]}$ &                   69.40±4.80 &                           57.88±6.15 &                            74.93$_{[3]}$±2.52 \\
Bitfit       &          55.26±\textbf{0.78}$_{[1]}$ &                           89.98±0.15 &                                    86.87±1.27 &                           71.36±1.71 &  \textbf{91.29}$_{[1]}$±\textbf{2.27}$_{[1]}$ &           71.80$_{[2]}$±3.92 &                           55.29±9.90 &                                    74.55±2.86 \\
MagPruning   &                           56.45±1.80 &  90.26$_{[2]}$±\textbf{0.11}$_{[1]}$ &                                    87.35±0.85 &                           72.24±2.14 &                                    84.46±3.58 &                   69.20±3.54 &  \textbf{59.71}$_{[1]}$±3.88$_{[2]}$ &                            74.24±2.27$_{[2]}$ \\
Adapter      &                           60.05±1.88 &                           89.92±0.19 &                   \textbf{88.79}$_{[1]}$±0.80 &                           74.55±1.80 &                                    86.61±4.97 &  68.80±\textbf{2.40}$_{[1]}$ &                           55.63±7.53 &                                    74.91±2.79 \\
LoRA         &  \textbf{61.46}$_{[1]}$±1.27$_{[3]}$ &                           86.73±0.38 &                                    88.28±1.06 &  \textbf{76.46}$_{[1]}$±1.34$_{[3]}$ &                            88.69$_{[3]}$±5.32 &           67.75±2.49$_{[2]}$ &           58.85$_{[3]}$±4.27$_{[3]}$ &                    75.46$_{[2]}$±2.30$_{[3]}$ \\
DiffPruning  &                           58.36±1.45 &                           89.52±0.27 &                                    77.46±5.31 &                           70.76±9.01 &                            85.18±2.65$_{[2]}$ &           70.40$_{[3]}$±3.07 &                           55.38±4.30 &                                    72.44±3.72 \\
ChildPruning &                           59.40±2.30 &                           89.33±3.23 &                                    88.43±0.80 &                           75.11±2.87 &                                    85.71±4.07 &                   70.30±4.54 &                           54.04±7.24 &                                    74.62±3.58 \\
SAM          &                   59.52±1.12$_{[2]}$ &  \textbf{90.45}$_{[1]}$±0.12$_{[2]}$ &  \textbf{88.79}$_{[1]}$±\textbf{0.69}$_{[1]}$ &  75.74$_{[2]}$±\textbf{1.27}$_{[1]}$ &                                    86.79±4.39 &  \textbf{74.00}$_{[1]}$±2.79 &  59.52$_{[2]}$±\textbf{3.32}$_{[1]}$ &  \textbf{76.40}$_{[1]}$±\textbf{1.96}$_{[1]}$ \\
  \bottomrule
  \end{tabular}

  \caption{Data perturbation stability. The setting is the same as the main experiments except that we run the experiments on different sampled datasets. Due to the space limit, we attach the significance test in Appendix \ref{sec:A-ttest}.}
  \label{tab:datastablity}
  \end{table*}

\textbf{Projection Discontinuity Problem.}
To give an intuitive illustration of the projection discontinuity problem in projection-based approaches, we plot the training curve of the DiffPruning method on the CB task. As illustrated in Fig. \ref{fig:cb-sgpa-600}, we adjust the mask every 600 training steps. It can be observed from the figure that each time we change the mask, the training error will go back to almost the same value as its initial loss. This result shows that changing the mask severely affects the training procedure due to the projection discontinuity problem.

\textbf{Relation between Stability and Overall Performance.} Theorem \ref{thm:generalization} shows that stability implies better generalization. To further validate this, we illustrate how the stability ranks and the overall performance ranks are correlated in the main experiment. As shown in Fig. \ref{fig:spearman}, the x-axis is the stability rank in each main experiment while the y-axis is the corresponding overall performance rank. For each vertical line of a specific stability rank, the dot indicates the overall performance mean rank value while the line length indicates the standard deviation. It can be observed from the figure that the two ranks are positively correlated indicating that stabler models usually have better generalization capability. To further show the relationship between the stability and the overall performance, we calculate Spearman's rank correlation coefficient \cite{spearman1904proof} for the two ranks. It can be denoted as $\rho=\frac{\operatorname{cov}(R(S),R(V))}{\sigma_{R(S)}\sigma_{R(V)}}$, where $R(S)$ and $R(V)$ are the rank variables, $\operatorname{cov}(R(S),R(V))$ is the covariance of $R(S)$ and $R(V)$ while $\sigma_{R(V)}$ is the standard deviation of the rank variable $V$. We have $\rho=0.4356$ with p-value$=0.000014<0.05$ indicating that the correlation between the two rank variables is significant. 

\textbf{Effectiveness of Sparsity.} To further verify our theoretical analysis in Theorem \ref{thm:stable} and Theorem \ref{thm:generalization}, we conduct a new experiment to show how the overall performance and the stability change as we change the sparsity. We change the sparsity of the SAM model in \{0.0002, 0.0005, 0.001, 0.002, 0.005, 0.01, 0.02, 0.05, 0.1, 0.2\} and plot the relationship between sparsity and the mean/standard deviation in both the test set and training set. The results are shown in Fig. \ref{fig:sparsity}. It can be concluded from the results that (1) as the sparsity ratio decreases, the mean and the standard deviation of most tasks also decrease which means the models become more stable with better generalization. This observation is consistent with our bound in Theorem \ref{thm:stable} and Theorem \ref{thm:generalization}. (2) If the sparsity ratio drops below a certain threshold, the models become quite unstable and the performance also sees a sharp drop. This is because the empirical error increases drastically which can be observed in the Train Mean and Train Std scores in Fig. \ref{fig:sparsity}. At the same time, under such circumstances, decreasing the sparsity ratio cannot further lower the bound effectively. Therefore, such observation is also consistent with our discussion in Theorem \ref{thm:stable} and Theorem \ref{thm:generalization}.

\textbf{Data Perturbation Stability.} In the main experiment, we use different random seeds. However, it is unknown whether the performance is still stable if we have a perturbation on the dataset. We conduct a new experiment to verify the data perturbation stability by training the model on 10 different training sets. Each of them is made by randomly removing 10\% training samples from our original training set. The results are shown in Table \ref{tab:datastablity}. It can be observed from the results that the data perturbation stability performance is similar to the main experiment and our proposed SAM model still has the best data perturbation stability as well as the overall performance among all the models.

\section{Related Works}

Fine-tuning on a pre-trained model \cite{Peters2018DeepCW,devlin2019bert,lan2020albert,radford2018improving,radford2019language,brown2020language,dong2019unified,qiu2020pre,chen2022revisiting,Liu2022FewShotPF} has shown to be very promising in recent years. However, fine-tuning the full model yields a large model with the same size for each task and many works indicate that fine-tuning the full model is unstable \cite{devlin2019bert,phang2018sentence,lee2019mixout,zhu2020freelb,dodge2020fine,pruksachatkun2020intermediate,mosbach2020stability,zhang2020revisiting,zhao2021calibrate}. To solve this problem, many researchers propose the parameter-efficient methods which only fine-tune a small part of the pre-trained parameters. These methods are found to be more stable than fine-tuning the full model \cite{he2021effectiveness,lee2019mixout,houlsby2019parameter,zaken2021bitfit,sung2021training,liu2021p}. Currently, there is still no previous work providing a theoretical analysis for the stability of the parameter-efficient models.

Depending on how the parameter-efficient models choose which parameters to optimize, we categorize them into 1) random approaches (Random and Mixout\cite{lee2019mixout}), 2) rule-based approaches (BitFit \cite{zaken2021bitfit}, MagPruning \cite{han2015deep,han2015learning,lee2021layeradaptive}, Adapter \cite{houlsby2019parameter,pfeiffer2020adapterhub,ruckle2021adapterdrop,he2021effectiveness}, LoRA \cite{hu2022lora}), and 3) projection-based approaches (DiffPruning \cite{mallya2018piggyback,sanh2020movement,guo2021parameter}, ChildPruning \cite{xu2021raise}). We refer readers to section \ref{sec:models} for more detailed discussion about these models. Despite the promising results achieved by these models, the random and rule-based approaches do not utilize the information from task-specific data while the projection-based approaches have the projection discontinuity problem.

Apart from the parameter efficient fine-tuning methods, many other approaches \cite{xuhong2018explicit,jiang2020smart,hua2021noise} have also been proposed to regularize the parameters to enhance the generalization capability. Moreover, many research works \cite{salman2020adversarially,jiang2020smart,li2021improved,hua2021noise} propose to train model adversarially while some researchers \cite{you2019drawing,liang2021finding,ansell2021composable,chen2021earlybert} propose to utilize the lottery ticket approaches to prune the network. Besides, the prompt-tuning \cite{liu2021p,lester2021power,li2021prefix} methods try to use prefix to adapt the model into new domains without changing the model parameters and the continuous prompts method \cite{li2021prefix} can be categorized into the rule-based approaches. Currently, our work focuses on approaches that only fine-tune a small part of the model which is very different from these models in structure or procedure.

\section{Conclusions}
In this paper, we propose to understand the effectiveness of the parameter-efficient fine-tuning models. Depending on how the tunable parameters are chosen, we first categorize most of the models into three categories, namely, random approaches, rule-based approaches, and projection-based approaches. Then, we show that all models in the three categories are sparse fine-tuned models and we give a theoretical analysis of the stability and the generalization error. We further show that the random approaches and the rule-based methods do not utilize the task data information while the projection-based approaches suffer from the projection discontinuity problem. We propose a novel SAM model to alleviate both problems and we conduct extensive experiments to show the correctness of our theoretical analysis and the effectiveness of our proposed models.

\section*{Acknowledgments}
The authors gratefully acknowledge the support of the funding from UKRI under project code ES/T012277/1.

\bibliographystyle{acl_natbib}
\bibliography{reference}

\begin{thebibliography}{88}
\expandafter\ifx\csname natexlab\endcsname\relax\def\natexlab#1{#1}\fi

\bibitem[{Ansell et~al.(2021)Ansell, Ponti, Korhonen, and
  Vuli{\'c}}]{ansell2021composable}
Alan Ansell, Edoardo~Maria Ponti, Anna Korhonen, and Ivan Vuli{\'c}. 2021.
\newblock Composable sparse fine-tuning for cross-lingual transfer.
\newblock \emph{arXiv preprint arXiv:2110.07560}.

\bibitem[{Bentivogli et~al.(2009)Bentivogli, Clark, Dagan, and
  Giampiccolo}]{bentivogli2009fifth}
Luisa Bentivogli, Peter Clark, Ido Dagan, and Danilo Giampiccolo. 2009.
\newblock The fifth pascal recognizing textual entailment challenge.
\newblock In \emph{TAC}.

\bibitem[{Bishop and Nasrabadi(2006)}]{bishop2006pattern}
Christopher~M Bishop and Nasser~M Nasrabadi. 2006.
\newblock \emph{Pattern recognition and machine learning}, volume~4.
\newblock Springer.

\bibitem[{Bollapragada et~al.(2019)Bollapragada, Byrd, and
  Nocedal}]{bollapragada2019exact}
Raghu Bollapragada, Richard~H Byrd, and Jorge Nocedal. 2019.
\newblock Exact and inexact subsampled newton methods for optimization.
\newblock \emph{IMA Journal of Numerical Analysis}, 39(2):545--578.

\bibitem[{Bousquet and Elisseeff(2002)}]{bousquet2002stability}
Olivier Bousquet and Andr{\'e} Elisseeff. 2002.
\newblock Stability and generalization.
\newblock \emph{The Journal of Machine Learning Research}, 2:499--526.

\bibitem[{Brown et~al.(2020)Brown, Mann, Ryder, Subbiah, Kaplan, Dhariwal,
  Neelakantan, Shyam, Sastry, Askell et~al.}]{brown2020language}
Tom Brown, Benjamin Mann, Nick Ryder, Melanie Subbiah, Jared~D Kaplan, Prafulla
  Dhariwal, Arvind Neelakantan, Pranav Shyam, Girish Sastry, Amanda Askell,
  et~al. 2020.
\newblock Language models are few-shot learners.
\newblock \emph{Advances in neural information processing systems},
  33:1877--1901.

\bibitem[{Cer et~al.(2017)Cer, Diab, Agirre, Lopez-Gazpio, and
  Specia}]{cer2017semeval}
Daniel Cer, Mona Diab, Eneko Agirre, Inigo Lopez-Gazpio, and Lucia Specia.
  2017.
\newblock Semeval-2017 task 1: Semantic textual similarity-multilingual and
  cross-lingual focused evaluation.
\newblock \emph{arXiv preprint arXiv:1708.00055}.

\bibitem[{Charles and Papailiopoulos(2018)}]{charles2018stability}
Zachary Charles and Dimitris Papailiopoulos. 2018.
\newblock Stability and generalization of learning algorithms that converge to
  global optima.
\newblock In \emph{International conference on machine learning}, pages
  745--754. PMLR.

\bibitem[{Chen et~al.(2022)Chen, Liu, Meng, and Liang}]{chen2022revisiting}
Guanzheng Chen, Fangyu Liu, Zaiqiao Meng, and Shangsong Liang. 2022.
\newblock Revisiting parameter-efficient tuning: Are we really there yet?
\newblock \emph{arXiv preprint arXiv:2202.07962}.

\bibitem[{Chen et~al.(2021)Chen, Cheng, Wang, Gan, Wang, and
  Liu}]{chen2021earlybert}
Xiaohan Chen, Yu~Cheng, Shuohang Wang, Zhe Gan, Zhangyang Wang, and Jingjing
  Liu. 2021.
\newblock Earlybert: Efficient bert training via early-bird lottery tickets.
\newblock In \emph{Proceedings of the 59th Annual Meeting of the Association
  for Computational Linguistics and the 11th International Joint Conference on
  Natural Language Processing (Volume 1: Long Papers)}, pages 2195--2207.

\bibitem[{Dagan et~al.(2005)Dagan, Glickman, and Magnini}]{dagan2005pascal}
Ido Dagan, Oren Glickman, and Bernardo Magnini. 2005.
\newblock The pascal recognising textual entailment challenge.
\newblock In \emph{Machine Learning Challenges Workshop}, pages 177--190.
  Springer.

\bibitem[{De~Marneffe et~al.(2019)De~Marneffe, Simons, and
  Tonhauser}]{de2019commitmentbank}
Marie-Catherine De~Marneffe, Mandy Simons, and Judith Tonhauser. 2019.
\newblock The commitmentbank: Investigating projection in naturally occurring
  discourse.
\newblock In \emph{proceedings of Sinn und Bedeutung}, volume~23, pages
  107--124.

\bibitem[{Dembo et~al.(1982)Dembo, Eisenstat, and Steihaug}]{dembo1982inexact}
Ron~S Dembo, Stanley~C Eisenstat, and Trond Steihaug. 1982.
\newblock Inexact newton methods.
\newblock \emph{SIAM Journal on Numerical analysis}, 19(2):400--408.

\bibitem[{Devlin et~al.(2019)Devlin, Chang, Lee, and
  Toutanova}]{devlin2019bert}
Jacob Devlin, Ming-Wei Chang, Kenton Lee, and Kristina Toutanova. 2019.
\newblock Bert: Pre-training of deep bidirectional transformers for language
  understanding.
\newblock In \emph{Proceedings of NAACL-HLT}, pages 4171--4186.

\bibitem[{Ding et~al.(2022)Ding, Qin, Yang, Wei, Yang, Su, Hu, Chen, Chan, Chen
  et~al.}]{ding2022delta}
Ning Ding, Yujia Qin, Guang Yang, Fuchao Wei, Zonghan Yang, Yusheng Su,
  Shengding Hu, Yulin Chen, Chi-Min Chan, Weize Chen, et~al. 2022.
\newblock Delta tuning: A comprehensive study of parameter efficient methods
  for pre-trained language models.
\newblock \emph{arXiv preprint arXiv:2203.06904}.

\bibitem[{Dodge et~al.(2020)Dodge, Ilharco, Schwartz, Farhadi, Hajishirzi, and
  Smith}]{dodge2020fine}
Jesse Dodge, Gabriel Ilharco, Roy Schwartz, Ali Farhadi, Hannaneh Hajishirzi,
  and Noah~A Smith. 2020.
\newblock Fine-tuning pretrained language models: Weight initializations, data
  orders, and early stopping.

\bibitem[{Dolan and Brockett(2005)}]{dolan2005automatically}
Bill Dolan and Chris Brockett. 2005.
\newblock Automatically constructing a corpus of sentential paraphrases.
\newblock In \emph{Third International Workshop on Paraphrasing (IWP2005)}.

\bibitem[{Dong et~al.(2019)Dong, Yang, Wang, Wei, Liu, Wang, Gao, Zhou, and
  Hon}]{dong2019unified}
Li~Dong, Nan Yang, Wenhui Wang, Furu Wei, Xiaodong Liu, Yu~Wang, Jianfeng Gao,
  Ming Zhou, and Hsiao-Wuen Hon. 2019.
\newblock Unified language model pre-training for natural language
  understanding and generation.
\newblock \emph{Advances in Neural Information Processing Systems}, 32.

\bibitem[{Elisseeff et~al.(2005)Elisseeff, Evgeniou, Pontil, and
  Kaelbing}]{elisseeff2005stability}
Andre Elisseeff, Theodoros Evgeniou, Massimiliano Pontil, and Leslie~Pack
  Kaelbing. 2005.
\newblock Stability of randomized learning algorithms.
\newblock \emph{Journal of Machine Learning Research}, 6(1).

\bibitem[{Fedus et~al.(2021)Fedus, Zoph, and Shazeer}]{fedus2021switch}
William Fedus, Barret Zoph, and Noam Shazeer. 2021.
\newblock Switch transformers: Scaling to trillion parameter models with simple
  and efficient sparsity.
\newblock \emph{arXiv preprint arXiv:2101.03961}.

\bibitem[{Fu et~al.(2021)Fu, Lam, So, and Shi}]{fu2021theoretical}
Zihao Fu, Wai Lam, Anthony Man-Cho So, and Bei Shi. 2021.
\newblock A theoretical analysis of the repetition problem in text generation.
\newblock In \emph{Proceedings of the AAAI Conference on Artificial
  Intelligence}, volume~35, pages 12848--12856.

\bibitem[{Giampiccolo et~al.(2007)Giampiccolo, Magnini, Dagan, and
  Dolan}]{giampiccolo2007third}
Danilo Giampiccolo, Bernardo Magnini, Ido Dagan, and William~B Dolan. 2007.
\newblock The third pascal recognizing textual entailment challenge.
\newblock In \emph{Proceedings of the ACL-PASCAL workshop on textual entailment
  and paraphrasing}, pages 1--9.

\bibitem[{Guo et~al.(2021)Guo, Rush, and Kim}]{guo2021parameter}
Demi Guo, Alexander~M Rush, and Yoon Kim. 2021.
\newblock Parameter-efficient transfer learning with diff pruning.
\newblock In \emph{Proceedings of the 59th Annual Meeting of the Association
  for Computational Linguistics and the 11th International Joint Conference on
  Natural Language Processing (Volume 1: Long Papers)}, pages 4884--4896.

\bibitem[{Han et~al.(2015{\natexlab{a}})Han, Mao, and Dally}]{han2015deep}
Song Han, Huizi Mao, and William~J Dally. 2015{\natexlab{a}}.
\newblock Deep compression: Compressing deep neural networks with pruning,
  trained quantization and huffman coding.
\newblock \emph{arXiv preprint arXiv:1510.00149}.

\bibitem[{Han et~al.(2015{\natexlab{b}})Han, Pool, Tran, and
  Dally}]{han2015learning}
Song Han, Jeff Pool, John Tran, and William Dally. 2015{\natexlab{b}}.
\newblock Learning both weights and connections for efficient neural network.
\newblock \emph{Advances in neural information processing systems}, 28.

\bibitem[{Hardt et~al.(2016)Hardt, Recht, and Singer}]{hardt2016train}
Moritz Hardt, Ben Recht, and Yoram Singer. 2016.
\newblock Train faster, generalize better: Stability of stochastic gradient
  descent.
\newblock In \emph{International conference on machine learning}, pages
  1225--1234. PMLR.

\bibitem[{He et~al.(2021{\natexlab{a}})He, Zhou, Ma, Berg-Kirkpatrick, and
  Neubig}]{he2021towards}
Junxian He, Chunting Zhou, Xuezhe Ma, Taylor Berg-Kirkpatrick, and Graham
  Neubig. 2021{\natexlab{a}}.
\newblock Towards a unified view of parameter-efficient transfer learning.
\newblock \emph{arXiv preprint arXiv:2110.04366}.

\bibitem[{He et~al.(2021{\natexlab{b}})He, Liu, Ye, Tan, Ding, Cheng, Low,
  Bing, and Si}]{he2021effectiveness}
Ruidan He, Linlin Liu, Hai Ye, Qingyu Tan, Bosheng Ding, Liying Cheng, Jiawei
  Low, Lidong Bing, and Luo Si. 2021{\natexlab{b}}.
\newblock On the effectiveness of adapter-based tuning for pretrained language
  model adaptation.
\newblock In \emph{Proceedings of the 59th Annual Meeting of the Association
  for Computational Linguistics and the 11th International Joint Conference on
  Natural Language Processing (Volume 1: Long Papers)}, pages 2208--2222.

\bibitem[{Houlsby et~al.(2019)Houlsby, Giurgiu, Jastrzebski, Morrone,
  De~Laroussilhe, Gesmundo, Attariyan, and Gelly}]{houlsby2019parameter}
Neil Houlsby, Andrei Giurgiu, Stanislaw Jastrzebski, Bruna Morrone, Quentin
  De~Laroussilhe, Andrea Gesmundo, Mona Attariyan, and Sylvain Gelly. 2019.
\newblock Parameter-efficient transfer learning for nlp.
\newblock In \emph{International Conference on Machine Learning}, pages
  2790--2799. PMLR.

\bibitem[{Hu et~al.(2022)Hu, yelong shen, Wallis, Allen-Zhu, Li, Wang, Wang,
  and Chen}]{hu2022lora}
Edward~J Hu, yelong shen, Phillip Wallis, Zeyuan Allen-Zhu, Yuanzhi Li, Shean
  Wang, Lu~Wang, and Weizhu Chen. 2022.
\newblock \href {https://openreview.net/forum?id=nZeVKeeFYf9} {Lo{RA}: Low-rank
  adaptation of large language models}.
\newblock In \emph{International Conference on Learning Representations}.

\bibitem[{Hua et~al.(2021)Hua, Li, Dou, Xu, and Luo}]{hua2021noise}
Hang Hua, Xingjian Li, Dejing Dou, Chengzhong Xu, and Jiebo Luo. 2021.
\newblock Noise stability regularization for improving bert fine-tuning.
\newblock In \emph{Proceedings of the 2021 Conference of the North American
  Chapter of the Association for Computational Linguistics: Human Language
  Technologies}, pages 3229--3241.

\bibitem[{Jiang et~al.(2020)Jiang, He, Chen, Liu, Gao, and
  Zhao}]{jiang2020smart}
Haoming Jiang, Pengcheng He, Weizhu Chen, Xiaodong Liu, Jianfeng Gao, and Tuo
  Zhao. 2020.
\newblock Smart: Robust and efficient fine-tuning for pre-trained natural
  language models through principled regularized optimization.
\newblock In \emph{Proceedings of the 58th Annual Meeting of the Association
  for Computational Linguistics}, pages 2177--2190.

\bibitem[{Karimi~Mahabadi et~al.(2021)Karimi~Mahabadi, Henderson, and
  Ruder}]{karimi2021compacter}
Rabeeh Karimi~Mahabadi, James Henderson, and Sebastian Ruder. 2021.
\newblock Compacter: Efficient low-rank hypercomplex adapter layers.
\newblock \emph{Advances in Neural Information Processing Systems}, 34.

\bibitem[{Kenton and Toutanova(2019)}]{kenton2019bert}
Jacob Devlin Ming-Wei~Chang Kenton and Lee~Kristina Toutanova. 2019.
\newblock Bert: Pre-training of deep bidirectional transformers for language
  understanding.
\newblock In \emph{Proceedings of NAACL-HLT}, pages 4171--4186.

\bibitem[{Kim et~al.(2021)Kim, Shum, Susanj, and Hilgart}]{kim2021revisiting}
Seungwon Kim, Alex Shum, Nathan Susanj, and Jonathan Hilgart. 2021.
\newblock Revisiting pretraining with adapters.
\newblock In \emph{Proceedings of the 6th Workshop on Representation Learning
  for NLP (RepL4NLP-2021)}, pages 90--99.

\bibitem[{Kuzborskij and Lampert(2018)}]{kuzborskij2018data}
Ilja Kuzborskij and Christoph Lampert. 2018.
\newblock Data-dependent stability of stochastic gradient descent.
\newblock In \emph{International Conference on Machine Learning}, pages
  2815--2824. PMLR.

\bibitem[{Lagunas et~al.(2021)Lagunas, Charlaix, Sanh, and
  Rush}]{lagunas2021block}
Fran{\c{c}}ois Lagunas, Ella Charlaix, Victor Sanh, and Alexander~M Rush. 2021.
\newblock Block pruning for faster transformers.
\newblock In \emph{Proceedings of the 2021 Conference on Empirical Methods in
  Natural Language Processing}, pages 10619--10629.

\bibitem[{Lan et~al.(2020)Lan, Chen, Goodman, Gimpel, Sharma, and
  Soricut}]{lan2020albert}
Zhenzhong Lan, Mingda Chen, Sebastian Goodman, Kevin Gimpel, Piyush Sharma, and
  Radu Soricut. 2020.
\newblock Albert: A lite bert for self-supervised learning of language
  representations.
\newblock In \emph{International Conference on Learning Representations}.

\bibitem[{Lee et~al.(2019)Lee, Cho, and Kang}]{lee2019mixout}
Cheolhyoung Lee, Kyunghyun Cho, and Wanmo Kang. 2019.
\newblock Mixout: Effective regularization to finetune large-scale pretrained
  language models.
\newblock In \emph{International Conference on Learning Representations}.

\bibitem[{Lee et~al.(2021)Lee, Park, Mo, Ahn, and Shin}]{lee2021layeradaptive}
Jaeho Lee, Sejun Park, Sangwoo Mo, Sungsoo Ahn, and Jinwoo Shin. 2021.
\newblock \href {https://openreview.net/forum?id=H6ATjJ0TKdf} {Layer-adaptive
  sparsity for the magnitude-based pruning}.
\newblock In \emph{International Conference on Learning Representations}.

\bibitem[{Lester et~al.(2021)Lester, Al-Rfou, and Constant}]{lester2021power}
Brian Lester, Rami Al-Rfou, and Noah Constant. 2021.
\newblock The power of scale for parameter-efficient prompt tuning.
\newblock In \emph{Proceedings of the 2021 Conference on Empirical Methods in
  Natural Language Processing}, pages 3045--3059.

\bibitem[{Levesque et~al.(2012)Levesque, Davis, and
  Morgenstern}]{levesque2012winograd}
Hector Levesque, Ernest Davis, and Leora Morgenstern. 2012.
\newblock The winograd schema challenge.
\newblock In \emph{Thirteenth international conference on the principles of
  knowledge representation and reasoning}.

\bibitem[{Lewis et~al.(2020)Lewis, Liu, Goyal, Ghazvininejad, Mohamed, Levy,
  Stoyanov, and Zettlemoyer}]{lewis2020bart}
Mike Lewis, Yinhan Liu, Naman Goyal, Marjan Ghazvininejad, Abdelrahman Mohamed,
  Omer Levy, Veselin Stoyanov, and Luke Zettlemoyer. 2020.
\newblock Bart: Denoising sequence-to-sequence pre-training for natural
  language generation, translation, and comprehension.
\newblock In \emph{Proceedings of the 58th Annual Meeting of the Association
  for Computational Linguistics}, pages 7871--7880.

\bibitem[{Li and Zhang(2021)}]{li2021improved}
Dongyue Li and Hongyang Zhang. 2021.
\newblock Improved regularization and robustness for fine-tuning in neural
  networks.
\newblock \emph{Advances in Neural Information Processing Systems}, 34.

\bibitem[{Li and Liang(2021)}]{li2021prefix}
Xiang~Lisa Li and Percy Liang. 2021.
\newblock Prefix-tuning: Optimizing continuous prompts for generation.
\newblock In \emph{Proceedings of the 59th Annual Meeting of the Association
  for Computational Linguistics and the 11th International Joint Conference on
  Natural Language Processing (Volume 1: Long Papers)}, pages 4582--4597.

\bibitem[{Liang et~al.(2021)Liang, Zhao, Wang, Qiu, and Li}]{liang2021finding}
Jianze Liang, Chengqi Zhao, Mingxuan Wang, Xipeng Qiu, and Lei Li. 2021.
\newblock Finding sparse structures for domain specific neural machine
  translation.
\newblock In \emph{Proceedings of the AAAI Conference on Artificial
  Intelligence}, volume~35, pages 13333--13342.

\bibitem[{Liu et~al.(2022)Liu, Tam, Muqeeth, Mohta, Huang, Bansal, and
  Raffel}]{Liu2022FewShotPF}
Haokun Liu, Derek Tam, Mohammed Muqeeth, Jay Mohta, Tenghao Huang, Mohit
  Bansal, and Colin Raffel. 2022.
\newblock Few-shot parameter-efficient fine-tuning is better and cheaper than
  in-context learning.

\bibitem[{Liu et~al.(2021)Liu, Ji, Fu, Du, Yang, and Tang}]{liu2021p}
Xiao Liu, Kaixuan Ji, Yicheng Fu, Zhengxiao Du, Zhilin Yang, and Jie Tang.
  2021.
\newblock P-tuning v2: Prompt tuning can be comparable to fine-tuning
  universally across scales and tasks.
\newblock \emph{arXiv preprint arXiv:2110.07602}.

\bibitem[{Liu et~al.(2019)Liu, Ott, Goyal, Du, Joshi, Chen, Levy, Lewis,
  Zettlemoyer, and Stoyanov}]{liu2019roberta}
Yinhan Liu, Myle Ott, Naman Goyal, Jingfei Du, Mandar Joshi, Danqi Chen, Omer
  Levy, Mike Lewis, Luke Zettlemoyer, and Veselin Stoyanov. 2019.
\newblock Roberta: A robustly optimized bert pretraining approach.

\bibitem[{Mahabadi et~al.(2021)Mahabadi, Ruder, Dehghani, and
  Henderson}]{mahabadi2021parameter}
Rabeeh~Karimi Mahabadi, Sebastian Ruder, Mostafa Dehghani, and James Henderson.
  2021.
\newblock Parameter-efficient multi-task fine-tuning for transformers via
  shared hypernetworks.
\newblock In \emph{Proceedings of the 59th Annual Meeting of the Association
  for Computational Linguistics and the 11th International Joint Conference on
  Natural Language Processing (Volume 1: Long Papers)}, pages 565--576.

\bibitem[{Mallya et~al.(2018)Mallya, Davis, and Lazebnik}]{mallya2018piggyback}
Arun Mallya, Dillon Davis, and Svetlana Lazebnik. 2018.
\newblock Piggyback: Adapting a single network to multiple tasks by learning to
  mask weights.
\newblock In \emph{Proceedings of the European Conference on Computer Vision
  (ECCV)}, pages 67--82.

\bibitem[{Mao et~al.(2021)Mao, Mathias, Hou, Almahairi, Ma, Han, Yih, and
  Khabsa}]{mao2021unipelt}
Yuning Mao, Lambert Mathias, Rui Hou, Amjad Almahairi, Hao Ma, Jiawei Han,
  Wen-tau Yih, and Madian Khabsa. 2021.
\newblock Unipelt: A unified framework for parameter-efficient language model
  tuning.
\newblock \emph{arXiv preprint arXiv:2110.07577}.

\bibitem[{Mosbach et~al.(2020)Mosbach, Andriushchenko, and
  Klakow}]{mosbach2020stability}
Marius Mosbach, Maksym Andriushchenko, and Dietrich Klakow. 2020.
\newblock On the stability of fine-tuning bert: Misconceptions, explanations,
  and strong baselines.
\newblock In \emph{International Conference on Learning Representations}.

\bibitem[{Mostafa and Wang(2019)}]{mostafa2019parameter}
Hesham Mostafa and Xin Wang. 2019.
\newblock Parameter efficient training of deep convolutional neural networks by
  dynamic sparse reparameterization.
\newblock In \emph{International Conference on Machine Learning}, pages
  4646--4655. PMLR.

\bibitem[{Panahi et~al.(2021)Panahi, Saeedi, and
  Arodz}]{panahi2021shapeshifter}
Aliakbar Panahi, Seyran Saeedi, and Tom Arodz. 2021.
\newblock Shapeshifter: a parameter-efficient transformer using factorized
  reshaped matrices.
\newblock \emph{Advances in Neural Information Processing Systems}, 34.

\bibitem[{Peters et~al.(2018)Peters, Neumann, Iyyer, Gardner, Clark, Lee, and
  Zettlemoyer}]{Peters2018DeepCW}
Matthew~E. Peters, Mark Neumann, Mohit Iyyer, Matt Gardner, Christopher Clark,
  Kenton Lee, and Luke Zettlemoyer. 2018.
\newblock Deep contextualized word representations.
\newblock In \emph{NAACL}.

\bibitem[{Pfeiffer et~al.(2020)Pfeiffer, R{\"u}ckl{\'e}, Poth, Kamath,
  Vuli{\'c}, Ruder, Cho, and Gurevych}]{pfeiffer2020adapterhub}
Jonas Pfeiffer, Andreas R{\"u}ckl{\'e}, Clifton Poth, Aishwarya Kamath, Ivan
  Vuli{\'c}, Sebastian Ruder, Kyunghyun Cho, and Iryna Gurevych. 2020.
\newblock Adapterhub: A framework for adapting transformers.
\newblock In \emph{Proceedings of the 2020 Conference on Empirical Methods in
  Natural Language Processing: System Demonstrations}, pages 46--54.

\bibitem[{Phang et~al.(2018)Phang, F{\'e}vry, and Bowman}]{phang2018sentence}
Jason Phang, Thibault F{\'e}vry, and Samuel~R Bowman. 2018.
\newblock Sentence encoders on stilts: Supplementary training on intermediate
  labeled-data tasks.

\bibitem[{Phang et~al.(2020)Phang, Yeres, Swanson, Liu, Tenney, Htut, Vania,
  Wang, and Bowman}]{phang2020jiant}
Jason Phang, Phil Yeres, Jesse Swanson, Haokun Liu, Ian~F. Tenney, Phu~Mon
  Htut, Clara Vania, Alex Wang, and Samuel~R. Bowman. 2020.
\newblock \texttt{jiant} 2.0: A software toolkit for research on
  general-purpose text understanding models.
\newblock \url{http://jiant.info/}.

\bibitem[{Pruksachatkun et~al.(2020)Pruksachatkun, Phang, Liu, Htut, Zhang,
  Pang, Vania, Kann, and Bowman}]{pruksachatkun2020intermediate}
Yada Pruksachatkun, Jason Phang, Haokun Liu, Phu~Mon Htut, Xiaoyi Zhang,
  Richard~Yuanzhe Pang, Clara Vania, Katharina Kann, and Samuel Bowman. 2020.
\newblock Intermediate-task transfer learning with pretrained language models:
  When and why does it work?
\newblock In \emph{Proceedings of the 58th Annual Meeting of the Association
  for Computational Linguistics}, pages 5231--5247.

\bibitem[{Qiu et~al.(2020)Qiu, Sun, Xu, Shao, Dai, and Huang}]{qiu2020pre}
Xipeng Qiu, Tianxiang Sun, Yige Xu, Yunfan Shao, Ning Dai, and Xuanjing Huang.
  2020.
\newblock Pre-trained models for natural language processing: A survey.
\newblock \emph{Science China Technological Sciences}, 63(10):1872--1897.

\bibitem[{Radford et~al.(2018)Radford, Narasimhan, Salimans, and
  Sutskever}]{radford2018improving}
Alec Radford, Karthik Narasimhan, Tim Salimans, and Ilya Sutskever. 2018.
\newblock Improving language understanding by generative pre-training.

\bibitem[{Radford et~al.(2019)Radford, Wu, Child, Luan, Amodei, Sutskever
  et~al.}]{radford2019language}
Alec Radford, Jeffrey Wu, Rewon Child, David Luan, Dario Amodei, Ilya
  Sutskever, et~al. 2019.
\newblock Language models are unsupervised multitask learners.
\newblock \emph{OpenAI blog}, 1(8):9.

\bibitem[{Radiya-Dixit and Wang(2020)}]{radiya2020fine}
Evani Radiya-Dixit and Xin Wang. 2020.
\newblock How fine can fine-tuning be? learning efficient language models.
\newblock In \emph{International Conference on Artificial Intelligence and
  Statistics}, pages 2435--2443. PMLR.

\bibitem[{Raffel et~al.(2020)Raffel, Shazeer, Roberts, Lee, Narang, Matena,
  Zhou, Li, and Liu}]{raffel2020exploring}
Colin Raffel, Noam Shazeer, Adam Roberts, Katherine Lee, Sharan Narang, Michael
  Matena, Yanqi Zhou, Wei Li, and Peter~J Liu. 2020.
\newblock Exploring the limits of transfer learning with a unified text-to-text
  transformer.
\newblock \emph{Journal of Machine Learning Research}, 21:1--67.

\bibitem[{Ricotti et~al.(1988)Ricotti, Ragazzini, and
  Martinelli}]{ricotti1988learning}
Lucio~Prina Ricotti, Susanna Ragazzini, and Giuseppe Martinelli. 1988.
\newblock Learning of word stress in a sub-optimal second order
  back-propagation neural network.
\newblock In \emph{ICNN}, volume~1, pages 355--361.

\bibitem[{Roemmele et~al.(2011)Roemmele, Bejan, and
  Gordon}]{roemmele2011choice}
Melissa Roemmele, Cosmin~Adrian Bejan, and Andrew~S Gordon. 2011.
\newblock Choice of plausible alternatives: An evaluation of commonsense causal
  reasoning.
\newblock In \emph{AAAI spring symposium: logical formalizations of commonsense
  reasoning}, pages 90--95.

\bibitem[{R{\"u}ckl{\'e} et~al.(2021)R{\"u}ckl{\'e}, Geigle, Glockner, Beck,
  Pfeiffer, Reimers, and Gurevych}]{ruckle2021adapterdrop}
Andreas R{\"u}ckl{\'e}, Gregor Geigle, Max Glockner, Tilman Beck, Jonas
  Pfeiffer, Nils Reimers, and Iryna Gurevych. 2021.
\newblock Adapterdrop: On the efficiency of adapters in transformers.
\newblock In \emph{Proceedings of the 2021 Conference on Empirical Methods in
  Natural Language Processing}, pages 7930--7946.

\bibitem[{Salman et~al.(2020)Salman, Ilyas, Engstrom, Kapoor, and
  Madry}]{salman2020adversarially}
Hadi Salman, Andrew Ilyas, Logan Engstrom, Ashish Kapoor, and Aleksander Madry.
  2020.
\newblock Do adversarially robust imagenet models transfer better?
\newblock \emph{Advances in Neural Information Processing Systems},
  33:3533--3545.

\bibitem[{Sanh et~al.(2020)Sanh, Wolf, and Rush}]{sanh2020movement}
Victor Sanh, Thomas Wolf, and Alexander Rush. 2020.
\newblock Movement pruning: Adaptive sparsity by fine-tuning.
\newblock \emph{Advances in Neural Information Processing Systems},
  33:20378--20389.

\bibitem[{Shalev-Shwartz and Ben-David(2014)}]{shalev2014understanding}
Shai Shalev-Shwartz and Shai Ben-David. 2014.
\newblock \emph{Understanding machine learning: From theory to algorithms}.
\newblock Cambridge university press.

\bibitem[{Shalev-Shwartz et~al.(2010)Shalev-Shwartz, Shamir, Srebro, and
  Sridharan}]{shalev2010learnability}
Shai Shalev-Shwartz, Ohad Shamir, Nathan Srebro, and Karthik Sridharan. 2010.
\newblock Learnability, stability and uniform convergence.
\newblock \emph{The Journal of Machine Learning Research}, 11:2635--2670.

\bibitem[{Spearman(1904)}]{spearman1904proof}
Charles Spearman. 1904.
\newblock The proof and measurement of association between two things.
\newblock \emph{The American journal of psychology}, 15(1):72--101.

\bibitem[{Sung et~al.(2021)Sung, Nair, and Raffel}]{sung2021training}
Yi-Lin Sung, Varun Nair, and Colin~A Raffel. 2021.
\newblock Training neural networks with fixed sparse masks.
\newblock \emph{Advances in Neural Information Processing Systems}, 34.

\bibitem[{Wang et~al.(2019)Wang, Pruksachatkun, Nangia, Singh, Michael, Hill,
  Levy, and Bowman}]{wang2019superglue}
Alex Wang, Yada Pruksachatkun, Nikita Nangia, Amanpreet Singh, Julian Michael,
  Felix Hill, Omer Levy, and Samuel Bowman. 2019.
\newblock Superglue: A stickier benchmark for general-purpose language
  understanding systems.
\newblock \emph{Advances in neural information processing systems}, 32.

\bibitem[{Wang et~al.(2018)Wang, Singh, Michael, Hill, Levy, and
  Bowman}]{wang2018glue}
Alex Wang, Amanpreet Singh, Julian Michael, Felix Hill, Omer Levy, and Samuel
  Bowman. 2018.
\newblock Glue: A multi-task benchmark and analysis platform for natural
  language understanding.
\newblock In \emph{Proceedings of the 2018 EMNLP Workshop BlackboxNLP:
  Analyzing and Interpreting Neural Networks for NLP}, pages 353--355.

\bibitem[{Warstadt et~al.(2019)Warstadt, Singh, and
  Bowman}]{warstadt2019neural}
Alex Warstadt, Amanpreet Singh, and Samuel~R Bowman. 2019.
\newblock Neural network acceptability judgments.
\newblock \emph{Transactions of the Association for Computational Linguistics},
  7:625--641.

\bibitem[{Wolf et~al.(2020)Wolf, Debut, Sanh, Chaumond, Delangue, Moi, Cistac,
  Rault, Louf, Funtowicz, Davison, Shleifer, von Platen, Ma, Jernite, Plu, Xu,
  Scao, Gugger, Drame, Lhoest, and Rush}]{wolf-etal-2020-transformers}
Thomas Wolf, Lysandre Debut, Victor Sanh, Julien Chaumond, Clement Delangue,
  Anthony Moi, Pierric Cistac, Tim Rault, Rémi Louf, Morgan Funtowicz, Joe
  Davison, Sam Shleifer, Patrick von Platen, Clara Ma, Yacine Jernite, Julien
  Plu, Canwen Xu, Teven~Le Scao, Sylvain Gugger, Mariama Drame, Quentin Lhoest,
  and Alexander~M. Rush. 2020.
\newblock \href {https://www.aclweb.org/anthology/2020.emnlp-demos.6}
  {Transformers: State-of-the-art natural language processing}.
\newblock In \emph{Proceedings of the 2020 Conference on Empirical Methods in
  Natural Language Processing: System Demonstrations}, pages 38--45, Online.
  Association for Computational Linguistics.

\bibitem[{Xu et~al.(2020)Xu, Roosta, and Mahoney}]{xu2020newton}
Peng Xu, Fred Roosta, and Michael~W Mahoney. 2020.
\newblock Newton-type methods for non-convex optimization under inexact hessian
  information.
\newblock \emph{Mathematical Programming}, 184(1):35--70.

\bibitem[{Xu et~al.(2021)Xu, Luo, Zhang, Tan, Chang, Huang, and
  Huang}]{xu2021raise}
Runxin Xu, Fuli Luo, Zhiyuan Zhang, Chuanqi Tan, Baobao Chang, Songfang Huang,
  and Fei Huang. 2021.
\newblock Raise a child in large language model: Towards effective and
  generalizable fine-tuning.
\newblock In \emph{Proceedings of the 2021 Conference on Empirical Methods in
  Natural Language Processing}, pages 9514--9528.

\bibitem[{Xuhong et~al.(2018)Xuhong, Grandvalet, and
  Davoine}]{xuhong2018explicit}
LI~Xuhong, Yves Grandvalet, and Franck Davoine. 2018.
\newblock Explicit inductive bias for transfer learning with convolutional
  networks.
\newblock In \emph{International Conference on Machine Learning}, pages
  2825--2834. PMLR.

\bibitem[{Yao et~al.(2021{\natexlab{a}})Yao, Gholami, Shen, Mustafa, Keutzer,
  and Mahoney}]{yao2021adahessian}
Zhewei Yao, Amir Gholami, Sheng Shen, Mustafa Mustafa, Kurt Keutzer, and
  Michael Mahoney. 2021{\natexlab{a}}.
\newblock Adahessian: An adaptive second order optimizer for machine learning.
\newblock In \emph{Proceedings of the AAAI Conference on Artificial
  Intelligence}, volume~35, pages 10665--10673.

\bibitem[{Yao et~al.(2021{\natexlab{b}})Yao, Xu, Roosta, and
  Mahoney}]{yao2021inexact}
Zhewei Yao, Peng Xu, Fred Roosta, and Michael~W Mahoney. 2021{\natexlab{b}}.
\newblock Inexact nonconvex newton-type methods.
\newblock \emph{Informs Journal on Optimization}, 3(2):154--182.

\bibitem[{You et~al.(2019)You, Li, Xu, Fu, Wang, Chen, Baraniuk, Wang, and
  Lin}]{you2019drawing}
Haoran You, Chaojian Li, Pengfei Xu, Yonggan Fu, Yue Wang, Xiaohan Chen,
  Richard~G Baraniuk, Zhangyang Wang, and Yingyan Lin. 2019.
\newblock Drawing early-bird tickets: Toward more efficient training of deep
  networks.
\newblock In \emph{International Conference on Learning Representations}.

\bibitem[{Zaken et~al.(2021)Zaken, Ravfogel, and Goldberg}]{zaken2021bitfit}
Elad~Ben Zaken, Shauli Ravfogel, and Yoav Goldberg. 2021.
\newblock Bitfit: Simple parameter-efficient fine-tuning for transformer-based
  masked language-models.
\newblock \emph{arXiv preprint arXiv:2106.10199}.

\bibitem[{Zhang et~al.(2020)Zhang, Wu, Katiyar, Weinberger, and
  Artzi}]{zhang2020revisiting}
Tianyi Zhang, Felix Wu, Arzoo Katiyar, Kilian~Q Weinberger, and Yoav Artzi.
  2020.
\newblock Revisiting few-sample bert fine-tuning.
\newblock In \emph{International Conference on Learning Representations}.

\bibitem[{Zhao et~al.(2021)Zhao, Wallace, Feng, Klein, and
  Singh}]{zhao2021calibrate}
Zihao Zhao, Eric Wallace, Shi Feng, Dan Klein, and Sameer Singh. 2021.
\newblock Calibrate before use: Improving few-shot performance of language
  models.
\newblock In \emph{International Conference on Machine Learning}, pages
  12697--12706. PMLR.

\bibitem[{Zhu et~al.(2020)Zhu, Cheng, Gan, Sun, Goldstein, and
  Liu}]{zhu2020freelb}
Chen Zhu, Yu~Cheng, Zhe Gan, Siqi Sun, Tom Goldstein, and Jingjing Liu. 2020.
\newblock Freelb: Enhanced adversarial training for natural language
  understanding.
\newblock In \emph{ICLR}.

\end{thebibliography}

\clearpage
\appendix
\onecolumn
\begin{center}
  {\LARGE \textbf{Appendix. Supplementary Material}}
\end{center}
\renewcommand{\thesection}{A.\arabic{section}}
\setcounter{theorem}{0}
\setcounter{lemma}{0}
\setcounter{corollary}{0}
\setcounter{proposition}{0}

\section{Proof of Proposition \ref{thm:toreg}}\label{sec:A-toreg}
\begin{proposition}
  \thmtoreg
\end{proposition}

\begin{proof}
  We write the Lagrangian of Problem (\ref{eqn:primal}) as:
  $$\bar{L}=\min_\theta \max_\lambda \mathcal{L}(\theta)+\lambda\|(I-M)(\theta-\theta^0)\|^2$$
  where $\lambda$ is the Lagrangian multiplier. Therefore, Problem (\ref{eqn:primal}) is equivalent to solving the following problem:
 
  $$\begin{aligned}
    \min_\theta \max_\lambda \mathcal{L}(\theta)+\lambda\|(I-M)(\theta-\theta^0)\|^2\\
    \ge \max_\lambda \min_\theta  \mathcal{L}(\theta)+\lambda\|(I-M)(\theta-\theta^0)\|^2\\
    \ge \min_\theta  \mathcal{L}(\theta)+\|(I-M)(\theta-\theta^0)\|^2\\
  \end{aligned}$$
\end{proof}

\section{Proof of Theorem \ref{thm:stable}}\label{sec:A-stable}

Given training dataset $S=\{z_1,\cdots,z_n\}$, $S^i=S\backslash z_i=\{z_1,\cdots,z_{i-1},z_{i+1},\cdots,z_n\}$. Loss function $\mathcal{L}(\theta)=\frac{1}{n}\sum_{i=1}^n\ell(z_i)$.

The following proof is derived from \cite{shalev2014understanding} where the original Lemma utilizes the convex assumption which is not guaranteed in the neural network. We use the Taylor expansion instead of the convex assumption which makes it more suitable for describing the neural network models. However, the cost is that the bound will be related to a constant that is determined by the specific shape around the local minima.

\begin{lemma}[]
  \label{lemma:hps}
  Assume that the loss function $\ell$ is $\rho$ -Lipschitz. $\A(S^i)$ closes to $\A(S)$. The Hessian matrix $\nabla^2\mathcal{L}(\A(S))$ at $\A(S)$ is a positive-semidefinite matrix with a singular value decomposition as $U\operatorname{diag}(\Lambda) U^{-1}$, $\Lambda=\{\Lambda_1,\cdots,\Lambda_m\}$ and $\Lambda_{min}=\min\{\Lambda_1,\cdots,\Lambda_m\}$.  Then, the learning algorithm $\A$ deﬁned by
  $$\A(S)=\arg\min_w \frac{1}{n}\sum_{i=1}^n\ell(z_i)+\lambda \|w-w_0\|^2$$
  has pointwise hypothesis stability $\epsilon$ with rate $\frac{2\rho^2 }{(\Lambda_{min}+2\lambda)n}$. It follows that:
  $$\E_{S,i\sim U(n)}[|\ell(\A(S^{i}),z_i)-\ell(\A(S),z_i)|]\le \frac{2\rho^2 }{(\Lambda_{min}+2\lambda)n}$$
\end{lemma}

\begin{proof}
  We denote $f_S(w)=\mathcal{L}(w)+\lambda \|w-w_0\|^2$, and 
  $\A(S)=\arg\min_w f_S(w)$. As $\A(S)$ minimize $f_S(w)$, we have $\nabla f_S(\A(S))=0$. 
  $\forall v$ close to $\A(S)$, we have:
  $$\begin{aligned}
    f_S(v)=f_S(\A(S))+&[v-\A(S)]^T\nabla f_S(\A(S))+\frac{1}{2}[v-\A(S)]^T\nabla^2f_S(\A(S))[v-\A(S)]\\
    f_S(v)=f_S(\A(S))+&\frac{1}{2}[v-\A(S)]^T\nabla^2f_S(\A(S))[v-\A(S)]\\
    f_S(v)-f_S(\A(S))&=\frac{1}{2}[v-\A(S)]^T\nabla^2f_S(\A(S))[v-\A(S)]\\
    f_S(v)-f_S(\A(S))&=\frac{1}{2}[v-\A(S)]^T\nabla^2_{w=A(S)}(\mathcal{L}(w)+\lambda \|w-w_0\|^2)[v-\A(S)]\\
    f_S(v)-f_S(\A(S))&=\frac{1}{2}[v-\A(S)]^T (\nabla^2\mathcal{L}(\A(S))+2\lambda I)[v-\A(S)]\\
    f_S(v)-f_S(\A(S))&=\frac{1}{2}[v-\A(S)]^T (U\operatorname{diag}(\Lambda) U^{-1}+2\lambda I)[v-\A(S)]\\
    f_S(v)-f_S(\A(S))&=\frac{1}{2}[v-\A(S)]^T (U(\operatorname{diag}(\Lambda) + 2\lambda I)U^{-1})[v-\A(S)]\\
    f_S(v)-f_S(\A(S))&=\frac{1}{2}[v-\A(S)]^T U(\operatorname{diag}(\Lambda + 2\lambda )U^{-1}[v-\A(S)]\\
    f_S(v)-f_S(\A(S))&=\frac{1}{2}[v-\A(S)]^T U\operatorname{diag}(\sqrt{\Lambda_1 + 2\lambda}, \cdots, \sqrt{\Lambda_m + 2\lambda}) \cdot \\ &\ \ \ \ \ \ \operatorname{diag}(\sqrt{\Lambda_1 + 2\lambda}, \cdots, \sqrt{\Lambda_m + 2\lambda}) U^{-1}[v-\A(S)]\\
    f_S(v)-f_S(\A(S))&=\frac{1}{2}[v-\A(S)]^T U\operatorname{diag}(\sqrt{\Lambda_1 + 2\lambda}, \cdots, \sqrt{\Lambda_m + 2\lambda}) U^{-1}\cdot \\ &\ \ \ \ \ \  U\operatorname{diag}(\sqrt{\Lambda_1 + 2\lambda}, \cdots, \sqrt{\Lambda_m + 2\lambda}) U^{-1}[v-\A(S)]\\
    f_S(v)-f_S(\A(S))&=\frac{1}{2}\|U\operatorname{diag}(\sqrt{\Lambda_1 + 2\lambda}, \cdots, \sqrt{\Lambda_m + 2\lambda}) U^{-1}[v-\A(S)]\|^2\\
    f_S(v)-f_S(\A(S))&\ge\frac{1}{2}(\Lambda_{min}+2\lambda)\|v-\A(S)\|^2\\
  \end{aligned}$$
  Then, if $n$ is large enough, by the definition of $f_S(w)$. $\forall u,v$, we have:
  $$\begin{aligned}
    f_S(v)-f_S(u)&=L_S(v) + \lambda \|v-w_0\|^2-(L_S(u) + \lambda \|u-w_0\|^2)\\
    &=L_{S^i}(v) + \lambda \|v-w_0\|^2-(L_{S^i}(u) + \lambda \|u-w_0\|^2) + \\ &\frac{\ell (v,z_i)-\ell (u,z_i)}{n}\\
  \end{aligned}$$
  Then, we choose $v=\A(S^i),u=\A(S)$. As $v$ minimizes $L_{S^i}(v)+\lambda\|v-w_0\|^2$, we have:
  $$f_S(\A(S^i))-f_S(\A(S))\le \frac{\ell (\A(S^i),z_i)-\ell (\A(S),z_i)}{n}$$
  $$\frac{1}{2}(\Lambda_{min}+2\lambda)\|\A(S^i)-\A(S)\|^2\le \frac{\ell (\A(S^i),z_i)-\ell (\A(S),z_i)}{n}$$
  As the loss function $\ell(\cdot, z_i)$ is $\rho-$Lipschitz, we have:
  \begin{equation}
    |\ell (\A(S^i),z_i)-\ell (\A(S),z_i)|\le \rho \|\A(S^i)-\A(S)\|
    \label{eqn:lipschitz}
  \end{equation}
  Then, we have:
  $$\begin{aligned}
    \frac{1}{2}(\Lambda_{min}+2\lambda)\|A(S^i)-\A(S)\|^2\le\frac{\rho \|A(S^i)-\A(S)\|}{n}\\
    \|A(S^i)-\A(S)\|\le\frac{2\rho }{(\Lambda_{min}+2\lambda)n}
  \end{aligned}$$
  Then, plug back into Equation (\ref{eqn:lipschitz}), we have:
  $$|\ell(\A(S^{(i)}),z_i)-\ell(\A(S),z_i)|\le \frac{2\rho^2 }{(\Lambda_{min}+2\lambda)n}.$$
  As this holds for any $S,i$ we immediately obtain:
  $$\E_{S,i\sim U(n)}[|\ell(\A(S^{(i)}),z_i)-\ell(\A(S),z_i)|]\le \frac{2\rho^2 }{(\Lambda_{min}+2\lambda)n}.$$
\end{proof}

\begin{theorem}[Stability] \thmstable
\end{theorem}

\begin{proof}
  $$\begin{aligned}
  \end{aligned}$$
  $$\begin{aligned}
    \E_M L_R &= \mathcal{L}(\theta)+\E \|(I-M)(\theta^0-\theta)\|^2\\
    &=\mathcal{L}(\theta)+\E \sum_{i=1}^m(1-M_{ii})^2(\theta^0_i-\theta_i)^2\\
    &=\mathcal{L}(\theta)+ \sum_{i=1}^m(\theta^0_i-\theta_i)^2\E(1-M_{ii})^2\\
    &=\mathcal{L}(\theta)+ \sum_{i=1}^m(\theta^0_i-\theta_i)^2(1-p)\\
    &=\mathcal{L}(\theta)+(1-p)\|(\theta^0-\theta)\|^2 \\
  \end{aligned}$$
  From Lemma \ref{lemma:hps}, we have:
  $$\E_{S,i\sim U(m)}[|\ell(\A(S^i),z_i)-\ell(\A(S),z_i)|]\le \frac{2\rho^2 }{(\Lambda_{min}+2(1-p))n}$$
\end{proof}

\section{Proof of Theorem \ref{thm:generalization}}\label{sec:A-generalization}

We first give a Lemma from Theorem 11 of \citet{bousquet2002stability}. It has an extended version to random algorithms in \citet{elisseeff2005stability}.

\begin{lemma}
  For any learning algorithm $\A$ with pointwise hypothesis stability $\beta$ with respect to a loss function such that $0\le c(y,y')\le C$, we have with probability $1-\delta$,
  $$R(\A,S)\le \hat{R}(\A,S)+\sqrt{\frac{C^2+12Cn\beta}{2n\delta}},$$
where, $c(y,y')=|y-y'|$ is an absolute loss function.
\end{lemma}

\begin{theorem}[Generalization] 
  \thmgeneralization
\end{theorem}

\section{Proof of Theorem \ref{thm:minapprox}}\label{sec:A-minapprox}

\begin{theorem}
  \thmapprox
\end{theorem}

\begin{proof}
  We denote $\mathcal{I}=\{i|\hat{M}_{ii}=1\}$,
  $\forall M\in \Omega=\{M|\|M\|_0=\lfloor mp\rfloor; M_{ij}=0,\forall i\ne j;   M_{ii}\in \{0,1\}\}$ , we have
  $$\begin{aligned}
    \inf_{\Delta\theta,M\in \Omega} \LL(\theta^0+M\Delta \theta) &=\inf_{\Delta\theta,M\in \Omega} \LL(\theta^0) + \nabla \LL(\theta^0)^{\mathrm T} M\Delta \theta + \frac{1}{2} (M\Delta \theta)^{\mathrm T} H M\Delta \theta\\
    & =\inf_{\Delta\theta,M\in \Omega} \LL(\theta^0) + \sum_{i=1}^mM_{ii}\nabla \LL(\theta^0)_i \Delta \theta_i + \frac{1}{2}\sum_{i=1}^mM_{ii} h_i\Delta\theta_i^2\\
    & =\inf_{\Delta\theta,M\in \Omega} \LL(\theta^0) + \sum_{i=1}^mM_{ii}( \frac{1}{2}h_i\Delta\theta_i^2 + \nabla \LL(\theta^0)_i \Delta \theta_i)\\
    & =\inf_{M\in \Omega}\LL(\theta^0) - \sum_{i=1}^m M_{ii}\frac{\nabla \LL(\theta^0)_i^2}{2h_i}\\
    & \ge\LL(\theta^0) - \sum_{i\in \mathcal{I}} \frac{\nabla \LL(\theta^0)_i^2}{2h_i}\ \ \ \ \ \ \ \ \ \ \ \text{(By the definition of $\hat{M}$ and $\mathcal{I}$)}\\
    &=\LL(\theta^0) - \sum_{i=1}^m\hat{M}_{ii} \frac{\nabla \LL(\theta^0)_i^2}{2h_i}\\
    &= \inf_{\Delta\theta} \LL(\theta^0+\hat{M}\Delta\theta)
  \end{aligned}$$
\end{proof}

\section{Training Time Analysis}\label{sec:A-timeexp}

To further analyze the computational cost of each method, we list the training time for different models as shown in Table \ref{tab:traintime}. It should be noted that as we adopt the early stop strategy, models stop training when they converge and the running step number may be different from each other. It can be concluded from the results that: (1) Adapter and LoRA model outperform the FullTuning model as they only tune only a few newly added parameters; (2) Other parameter-efficient models spend more time than the FullTuning model because in the current implementation, all of these models use a mask to determine which parameters to tune and thus cannot save the computational cost. (3) Our proposed SAM model outperforms all other models (except Adapter and LoRA) as it has a faster convergence rate. However, it still needs a mask and thus needs more time to train than the FullTuning model. It should be noted that as indicated by \cite{guo2021parameter,houlsby2019parameter}, though most of the parameter-efficient models need more time to train, they actually need less storage space as they can just only store the changed parameters. This is quite useful when there are a lot of downstream tasks.

\begin{table}[H]
  \centering
  \scriptsize
  
  \begin{tabular}{lllllllll}
  \toprule
  {} &   CoLA &  MRPC & STS-B &   RTE &    CB &   COPA &    WSC &   AVG \\
  \midrule
  FullTuning   &   0.74 &   1.04 &   1.90 &   2.63 &   3.18 &   0.66 &   1.55 &   1.67 \\
ChildPruning &   3.91 &   1.36 &   2.12 &   3.77 &   3.13 &   0.91 &   2.34 &   2.51 \\
Adapter      &   0.89 &   1.24 &   3.43 &   3.51 &   1.16 &   0.41 &   0.47 &   1.59 \\
LoRA         &   1.39 &   1.23 &   2.12 &   4.87 &   2.15 &   0.60 &   1.43 &   1.97 \\
Bitfit       &   2.30 &   1.70 &   6.70 &   7.02 &   2.18 &   1.20 &   0.87 &   3.14 \\
DiffPruning  &   0.60 &   1.07 &   6.61 &   5.30 &   2.86 &   1.21 &   1.06 &   2.67 \\
Random       &   0.41 &   3.36 &   4.34 &   6.70 &   1.60 &   1.27 &   0.69 &   2.62 \\
SAM          &   2.31 &   1.68 &   2.78 &   3.15 &   2.70 &   0.72 &   0.82 &   2.02 \\
  \bottomrule
  \end{tabular}

  \caption{Training time (in hour) analysis.}
  \label{tab:traintime}
  \end{table}

\section{Tunable Parameter Ratio Comparison}\label{sec:A-params}
To make the model results comparable with each other, we try our best to set the tunable parameter ratio for each experiment as close as possible. Table \ref{tab:paraportion} shows the ratio of tunable parameters for each model. The tunable parameters for Adapter and BitFit are fixed and cannot be changed. We follow the official setting for the LoRA model to set the size of the tunable parameters. For other models including our proposed SAM model, we follow \citet{guo2021parameter} to set the tunable ratio as 0.5\% to make the models comparable with each other.

  \begin{table}[H]
    \centering
    \scriptsize
    
    \begin{tabular}{lllllllllll}
    \toprule
    {} & FullTuning & ChildPruning & Adapter & LoRA & Bitfit & DiffPruning & Random & MixOut & MagPruning &  SAM \\
    \midrule
    \%Param &        100 &          0.5 &    0.72 &  0.91 &   0.09 &         0.5 &    0.5 &    0.5 &        0.5 &  0.5 \\
    \bottomrule
    \end{tabular}

    \caption{Tunable parameter ratio comparison.}
    \label{tab:paraportion}
    \end{table}

\section{T-Test of the Significance of SAM}\label{sec:A-ttest}
To show how significantly our proposed SAM model outperforms other models, we conduct a t-test by comparing the SAM's results with other models' results. If the SAM model outperforms another model, the t-statistics will be greater than 0. If it outperforms another model significantly, the p-value will be less than 0.05. We report the t-statistics/p-value for the main experiment in Table \ref{tab:ttest-seed} as well as the corresponding results for the data perturbation experiments in Table \ref{tab:ttest-data}. We can conclude from the experimental results that our proposed SAM model outperforms the corresponding model significantly with p-values < 0.05. It shows that SAM model can outperform the other models in most of the cases. 

    \begin{table}[H]
      \centering
      \scriptsize
      
      \begin{tabular}{llllllll}
      \toprule
      {} &        CoLA &       STS-B &        MRPC &         RTE &          CB &       COPA &        WSC \\
      \midrule
      FullTuning   &   \TCG{3.81/0.00} &   \TCG{4.43/0.00} &       -1.69/0.11 &        0.83/0.42 &       -0.00/1.00 &  \TCG{3.36/0.00} &  \TCG{2.91/0.01} \\
Random       &   \TCG{5.38/0.00} &  \TCG{12.34/0.00} &        0.60/0.56 &  \TCG{3.35/0.00} &       -1.26/0.22 &  \TCG{3.47/0.00} &  \TCG{3.13/0.01} \\
MixOut       &   \TCG{3.07/0.01} &   \TCG{5.81/0.00} &        0.83/0.42 &       -0.96/0.35 &        1.60/0.13 &        1.47/0.16 &  \TCG{2.84/0.01} \\
Bitfit       &   \TCG{7.27/0.00} &   \TCG{6.89/0.00} &  \TCG{4.68/0.00} &  \TCG{4.04/0.00} &        1.62/0.12 &        1.73/0.10 &  \TCG{2.39/0.03} \\
MagPruning   &   \TCG{4.88/0.00} &   \TCG{4.34/0.00} &  \TCG{2.43/0.03} &  \TCG{3.78/0.00} &  \TCG{5.88/0.00} &        1.83/0.08 &  \TCG{2.84/0.01} \\
Adapter      &  \TCR{-2.34/0.03} &   \TCG{7.93/0.00} &       -1.12/0.28 &       -0.16/0.87 &        0.98/0.34 &  \TCG{3.07/0.01} &        1.60/0.13 \\
LoRA         &         0.02/0.99 &  \TCG{18.99/0.00} &       -1.87/0.08 &       -0.21/0.84 &  \TCG{3.10/0.01} &  \TCG{2.85/0.01} &        1.76/0.10 \\
DiffPruning  &   \TCG{4.00/0.00} &   \TCG{8.16/0.00} &  \TCG{4.97/0.00} &  \TCG{2.55/0.02} &  \TCG{2.53/0.02} &        1.38/0.18 &  \TCG{3.78/0.00} \\
ChildPruning &         1.67/0.11 &         1.27/0.22 &        1.35/0.19 &        0.66/0.52 &        1.90/0.07 &  \TCG{2.63/0.02} &  \TCG{2.59/0.02} \\
      \bottomrule
      \end{tabular}

      \caption{T-Test results (t-statistics/p-value) for comparing SAM with other models in each tasks. We use \TCG{green} color if the t-statistics is larger than 0 while the p-value < 0.05. It shows the SAM model outperforms the corresponding model significantly. We use \TCR{red} color if the t-statistics is less than 0 while the p-value < 0.05. It shows the corresponding model outperforms the SAM model significantly. The uncolored text means there is no statistical significance with respect to p-value=0.05.}
      \label{tab:ttest-seed}
      \end{table}

      \begin{table}[H]
        \centering
        \scriptsize
        
        \begin{tabular}{llllllll}
        \toprule
        {} &        CoLA &       STS-B &        MRPC &        RTE &          CB &       COPA &        WSC \\
        \midrule
        FullTuning   &       -0.25/0.81 &   \TCG{3.61/0.00} &        0.12/0.90 &        1.14/0.27 &         1.23/0.23 &  \TCG{3.26/0.00} &        1.76/0.10 \\
Random       &  \TCG{3.72/0.00} &   \TCG{7.99/0.00} &        0.66/0.52 &  \TCG{3.68/0.00} &        -1.14/0.27 &  \TCG{2.96/0.01} &  \TCG{2.49/0.02} \\
MixOut       &        0.06/0.95 &   \TCG{5.33/0.00} &        0.82/0.42 &  \TCG{2.47/0.02} &         1.69/0.11 &  \TCG{2.48/0.02} &        0.78/0.44 \\
Bitfit       &  \TCG{6.34/0.00} &   \TCG{6.65/0.00} &  \TCG{3.99/0.00} &  \TCG{5.64/0.00} &  \TCR{-2.49/0.02} &        1.37/0.19 &        1.27/0.22 \\
MagPruning   &  \TCG{3.39/0.01} &   \TCG{3.22/0.00} &  \TCG{3.96/0.00} &  \TCG{4.34/0.00} &         1.23/0.23 &  \TCG{3.19/0.01} &        0.05/0.96 \\
Adapter      &        0.32/0.75 &   \TCG{6.82/0.00} &       -0.02/0.99 &  \TCG{2.14/0.05} &         0.08/0.94 &  \TCG{6.23/0.00} &  \TCG{2.64/0.02} \\
LoRA         &       -1.09/0.30 &  \TCG{26.73/0.00} &        1.21/0.24 &        0.00/1.00 &        -0.72/0.48 &  \TCG{4.67/0.00} &        0.49/0.63 \\
DiffPruning  &        2.04/0.06 &   \TCG{9.16/0.00} &  \TCG{6.34/0.00} &        1.85/0.08 &         0.94/0.36 &  \TCG{2.60/0.02} &  \TCG{2.24/0.04} \\
ChildPruning &        0.74/0.48 &         1.05/0.31 &        1.02/0.32 &        1.16/0.26 &         0.54/0.60 &        2.08/0.05 &        2.09/0.05 \\
        \bottomrule
        \end{tabular}

        \caption{T-Test for significance of SAM in data perturbation stability experiments. The color scheme is the same as Table \ref{tab:ttest-seed}.}
      \label{tab:ttest-data}
        \end{table}

\section{Limitations and Future Directions}\label{sec:A-limit}
In this work, we theoretically prove the SAM model achieves approximal optimal value with Theorem \ref{thm:minapprox} with the second-order approximation. Therefore, there is still some room for improvement for the real target function. We can consider exploring some other assumptions of the target function like quadratic growth, Polyak-Łojasiewicz, etc. We may get different approximate optimal solutions under different assumptions. Besides, though the current model achieves better results with better stability, the training time is a little bit longer than the full tuning model. This is because, in the current implementation, the sparsity is achieved with a gradient mask. As a result, the training time may be slightly longer than an ordinary model. We can explore to further improve the implementation strategy to accelerate the running speed.

\section{Broader Impact Statement}\label{sec:ethic}
This paper proposes a theoretical analysis of existing methods and proposes an improved method under the same task setting. It may help researchers to understand existing models better and we also propose the SAM model to improve the model performance. The task is widely studied in the NLP community and we do not conduct any experiments with any living beings. Our work also does not cause any kind of safety or security concerns. Our work also does not raise any human rights concerns or environmental concerns. Therefore, there will be no negative societal impact on our work.
All the data used in this paper come from widely used datasets and we have given a detailed description and source of the datasets. As far as we know, these datasets do not have any personally identifiable information. From the previous literature, no bias cases were reported.

\end{document}